\documentclass[twoside,11pt]{article}

\usepackage[preprint]{jmlr2e}
\usepackage{xcolor}
\usepackage{graphicx,subfigure}
\usepackage{amssymb}
\usepackage{epstopdf}
\usepackage{mathtools, xfrac}
\usepackage{url}
\usepackage[linesnumbered,ruled,vlined]{algorithm2e}

\usepackage[applemac]{inputenc}

\def\L{\mathcal{L}}
\def\l{\ell}
\def\R{\mathcal{R}}
\def\S{\mathcal{S}}
\def\I{\mathcal{I}}

\title{Learning The Best Expert Efficiently }

\author{\name Daron Anderson \email andersd3@tcd.ie \\
       \addr Department of Computer Science and Statistics\\
       Trinity College Dublin\\
       Ireland
       \AND
       \name Douglas J. Leith \email doug.leith@tcd.ie \\
       \addr Department of Computer Science and Statistics\\
       Trinity College Dublin\\
       Ireland
       }

\editor{} 

\begin{document}
\maketitle

\begin{abstract}
We consider online learning problems where the aim is to achieve regret which is efficient in the sense that it is the same order as the lowest regret amongst $K$ experts.  This is a substantially stronger requirement that achieving $O(\sqrt{n})$ or $O(\log n)$ regret with respect to the best expert and standard algorithms are insufficient, even in easy cases where the regrets of the available actions are very different from one another.  We show that a particular lazy form of the online subgradient algorithm can be used to achieve minimal regret in a number of ``easy'' regimes while retaining an $O(\sqrt{n})$ worst-case regret guarantee.   We also show that for certain classes of problem minimal regret strategies exist for some of the remaining ``hard'' regimes. 
\end{abstract}


\section{Introduction}

We consider online convex optimisation in the \emph{efficient regret} setting.   By the efficient regret setting we mean that our task is to choose a sequence of actions such that the regret is of the same order as the lowest regret amongst $K$ experts.   So if, for example, the regret of the best expert is $O(1)$ then we want to actually achieve $O(1)$ regret.   This is, of course, much stronger than the usual requirement of $O(\sqrt{n})$ or $O(\log n)$ regret with respect to the best expert.

Our interest is motivated by applications such as the following.  Suppose a person has to make a choice each day, for example what time to leave for work in the morning.   Each day the person can use their insight,  e.g. gained from experience or information from friends, to propose a time.  The person is subject to behavioural biases as well as limited time and effort.   In addition, suppose a recommender system is available that each day proposes a time that comes with an $O(\sqrt{n})$ regret guarantee.   Our task each day is to decide between these two proposed times (or perhaps a combination of them) in such a way that the recommender provides a ``safety net''.  That is, if the person's proposed times have consistently lower regret than those proposed by the recommender then we want to achieve this lower regret.  But if the person's judgement is poor and the regret of their choices is greater than $O(\sqrt{n})$, then we want to fall back to the $O(\sqrt{n})$ regret of the recommender system.

Intuitively, there are two easy cases where we might reasonably hope to achieve efficient regret.  The first is where the difference in the regrets of the two experts is, in some appropriate sense, large.   For example, one expert has $\Theta(1)$ regret and the other $\Theta(\sqrt{n})$ regret.  Perhaps surprisingly, it is easy to come up with examples where standard online learning algorithms fail to achieve $O(1)$ regret in this case.   The second easy case is where both experts have similar regret, e.g. both have $\Theta(1)$ regret.  Unfortunately, again it is easy to come up with examples where standard algorithms fail to achieve $O(1)$ regret even in this case.   

In this paper we show that a particular form of the online subgradient algorithm, namely the Biased Lazy Subgraduent algorithm, can be used to achieve efficient regret in such easy cases while retaining an $O(\sqrt{n})$ worst-case regret guarantee.   This is not the standard greedy form of algorithm but rather a lazy subgradient method with varying step-size.   The remaining harder cases correspond to situations where there is no  consistent ordering of the regrets of the two experts or where the difference in their regrets is $\Theta(\log{n})$ or less.   We show that for certain classes of expert efficient regret strategies also exist for some of these harder cases.

\subsection{Related Work}

There are two main strands of related work.  The first, initiated by \cite{cesa-bianchi_improved_2007}, seeks better regret bounds in the low loss and i.i.d. stochastic regimes via second-order regret inequalities.   \cite{cesa-bianchi_improved_2007} derives two main types of second-order inequality.  One is of the form $\R_{n}\le \frac{\log K}{\eta}+\min_{k\in\{1,\dots,K\}} \eta\sum_{i=1}^n \l^2_{k,i}$ (translating to the loss setting), where $\R_{n}$ denotes the regret after $n$ steps, $K$ is the number of experts and $\l_{k,i}$ is the loss incurred by taking the action of expert $k$ at step $i$.   Since $\l^2_{k,i}<|\l_{k,i}|$ when the loss is small this improves on earlier bounds in the low loss regime.  The second type of inequality obtained is of the form $\R_{n}\le \sqrt{\log(K)\sum_{i=1}^n v_i}$ (again translating to the loss setting and also ignoring minor terms), where $v_{n}=\max_{i\le n} \min_{k\in\{1,\dots,K\}} \sum_{j=1}^i \l^2_{k,j}$ for the Prod algorithm and $v_i=\sum_{k=1}^Kp_{k,i}\l_{k,i}^2-(\sum_{k=1}^Kp_{k,i}\l_{k,i})^2$ for the Hedge algorithm with adaptive step size, where $p_{k,i}$ is the weight assigned to expert $k$ at step $i$.   \cite{gaillard_second-order_2014} build upon this to obtain regret inequalities of the form $\R_{n}\le \min_{k\in\{1,\dots,K\}} \sqrt{\log(K) \sum_{i=1}^n (\hat{\l}_i-\l_{k,i})^2}$ where $\hat{\l}_{i}=p_i^T\l_i$.  Using these they also obtain bounds for the low loss regime and also for i.i.d stochastic losses.  \cite{wintenberger_optimal_2017} and \cite{koolen_second-order_2015} take a different approach and obtain second order inequalities by modifying the Hedge algorithm to include a second order loss term.  A similar idea is also used by \cite{metagrad}.  

The low loss regime is not the same as the efficient regret regime, hence results for the low loss regime are of limited help in the efficient regret setting of interest in the present paper.  Second-order inequalities based on the deviation $\sum_{i=1}^n (\hat{\l}_i-\l_{k,i})^2$, or similar, can be expected to yield strong lower bounds when an algorithm quickly settles on a single expert.  Unfortunately, that leaves open the question of establishing conditions under which such rapid convergence takes place which, as we will see, turns out to be the key issue.

The second main strand of related work aims to construct so-called universal algorithms or algorithms achieving the ``best of both worlds''.  That is, a single algorithm that simultaneously achieves good regret in both the adversarial and stochastic settings, removing the need for prior knowledge of the setting when choosing the algorithm.   One strategy for achieving this is to 
start off using an algorithm suited to stochastic losses and then switch irreversibly to use of an adversarial algorithm if evidence accumulates that the stochastic assumption is false.  The other main strategy is to use reversible switches, with the decision as to which algorithm (or combination of algorithms) is used being updated in an online fashion.   One such strategy, the $(A,B)$-Prod algorithm introduced by \cite{ABprod}, is probably the closest approach in the literature to that considered in the present paper and is discussed in more detail in Section \ref{sec:prodhedge}.   Note that this work seeking universal algorithms by combining two specialised algorithms has perhaps been superceded by recent results showing that the Hedge and Subgradient algorithms with $\Theta(1/\sqrt{n})$ step size are in fact universal in this sense  (see \cite{mourtada_optimality_2019,anderson2019optimality}, respectively).

A related line of work uses the fact that popular algorithms such as Hedge can achieve good regret if the step size is tuned to the setting of interest, e.g. a step size of $\Theta(1/n)$ yields log regret for strongly convex losses.  The approach taken is therefore to try to learn the best step size in an online fashion.   See, for example, \cite{erven_adaptive_2011} and \cite{metagrad}.   

A third recent strand of related work addresses combining learning algorithms in the bandit setting.  \cite{agarwal_corralling_2017} and \cite{singla_learning_2018} consider combining time-varying experts with the aim of minimising regret with respect to the best constant action (referred to as ``competing with the best expert'').   Bandit setting aside, the setup is otherwise quite similar to that considered in the present paper.  The approach adopted is to manipulate the time-varying experts by adjusting in an online fashion the loss feedback provided to each expert.  Regret performance of $O(n^{2/3})$ is achieved when the best expert has $O(\sqrt{n})$ regret, and $O(\sqrt{n})$ when the best expert has $O(1)$ regret.

\section{Preliminaries}
We start with the usual online setup where at each step $i\in\{1,2,\dots\}$ we take action $y_i\in X\subset\mathbb{R}^m$, where $X$ is convex, closed and bounded, then observe vector $\l_i\in\mathbb{R}^m$ and suffer loss $\l_i^Ty_i$.  While we focus on linear losses $\l_i$ the extension to convex losses is immediate by the standard subgradient bounding method.   

Now suppose that at step $i$ we are restricted to choose amongst a set of $d$ actions $z_{k,i}\in X$, $k=1,2,\dots,d$.   For example, action $z_{1,i}$ may be proposed by a human and action $z_{2,i}$ by an opimisation algorithm.    That is, we are restricted to choosing a meta-action $x_i\in \S$, where $\S$ is the $d$-simplex, with meta-action $x_i\in \S$ corresponding to action $y_i=\sum_{k=1}^dz_{k,i}x_{k,i} \in X$, where $x_{k,i}$ denotes the $k$'th element of vector $x_i$.   Defining $b_i=(\l_i^Tz_{1,i}, \dots, \l_i^Tz_{d,i})$ then $b_i^Tx_i=\l_i^Ty_i$ and so the loss associated with meta-action $x_i$ is $b_i^Tx_i$. {For simplicity we assume all $\|b_i\| \le 1 $ where $\|\cdot\|$ is the Euclidean norm. The methods here immediately generalise to when we have a uniform bound $\|b_i\| \le L $ by a simple rescaling.}

The regret of a sequence of actions $y_i$, $i=1,\dots,n$ with respect to the best fixed action in $X$ is $\R_n=\sum_{i=1}^n \l_i^T \left(y_i- y^*\right)$, where $y^*\in\arg\min_{y\in X} \sum_{i=1}^n \l_i^T y$.   Substituting for $b_i$ and $x_i$ we have
\begin{align}
\R_n = \sum_{i=1}^n  \left(b_i^Tx_i- \l_i^Ty^*\right)\notag
\end{align}
We can also define the regret of $x_i$, $i=1,\dots,n$ with respect to the best fixed meta-action in $\S$, namely
\begin{align}
\tilde{\R}_n= \sum_{i=1}^n(b_i^Tx_i-b_i^Tx^*)\notag
\end{align} 
where $x^*\in\arg\min_{x\in \S} \sum_{i=1}^n b_i^T x$.   Since $\min_{x\in \S} \sum_{i=1}^n b_i^T x$ is a linear programme $x^*$ is an extreme point of the simplex.   That is, $x^*=e_{k^*}$ where $k^*\in\arg\min_{k\in\{1,\dots,d\}} \sum_{i=1}^n b_i^T e_k$ and $e_k$ denotes the unit vector with all elements zero apart from the $k$'th element which is equal to one.

Observe that in general $\R_n\ne \tilde{\R}_n$.  Indeed,
\begin{align}
\R_n=\sum_{i=1}^n \left(b_i^Tx^* - \l_i^Ty^*\right) + \sum_{i=1}^n b_i^T(x_i-x^*)\notag
\stackrel{(a)}{=} \min\{\R_{1,n},\dots,\R_{d,n}\}+ \tilde{\R}_n 
\end{align}
where $\R_{k,n}=\sum_{i=1}^n \left(\l_i^Tz_{k,i} - \l_i^Ty^*\right)=\sum_{i=1}^n(b_i^Te_k- \l_i^Ty^*)$ is the regret of the $k$'th expert and equality $(a)$ follows from the fact that 
\begin{align*}
x^*\in\arg\min_{x\in\S} \sum_{i=1}^n b_i^T x=\arg\min_{k\in\{1,\dots,d\}} \sum_{i=1}^n (b_i^T e_k- \l_i^Ty^*)
\end{align*} 
since $\sum_{i=1}^n\l_i^Ty^*$ is a constant that does  not depend on $x$. Our interest is in selecting a sequence $x_i$ such that ${\R}_n$ has order no greater than $\min\{\R_{1,n},\dots,\R_{d,n}\}$ i.e. ${\R}_n/\min\{\R_{1,n},\dots,\R_{d,n}\}$ is $O(1)$.  We refer to sequences with this property as having \emph{efficient regret}, or in short as being \emph{efficient}.   

\begin{figure}
\centering
\subfigure[Hedge]{
\includegraphics[width=0.45\columnwidth]{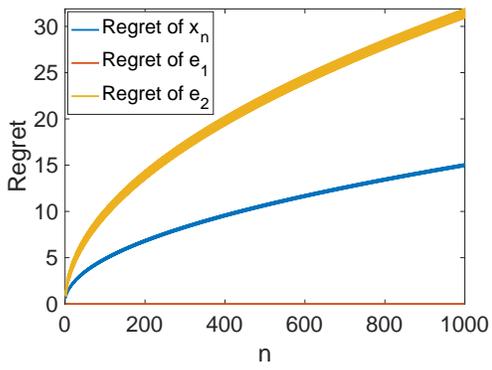}
}
\subfigure[Hedge]{
\includegraphics[width=0.45\columnwidth]{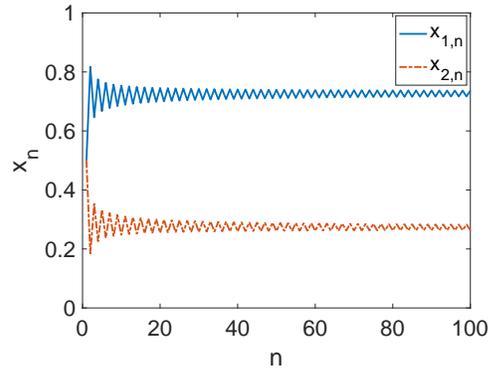}
}
\subfigure[Greedy Subgradient]{
\includegraphics[width=0.45\columnwidth]{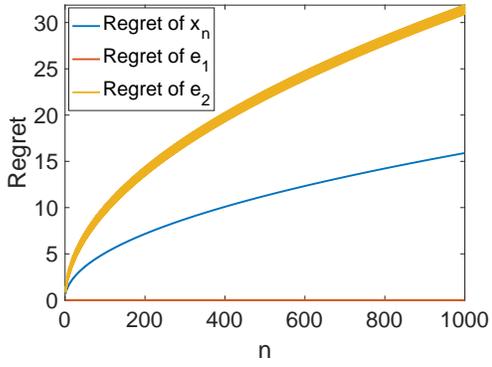}
}
\subfigure[Greedy Subgradient]{
\includegraphics[width=0.45\columnwidth]{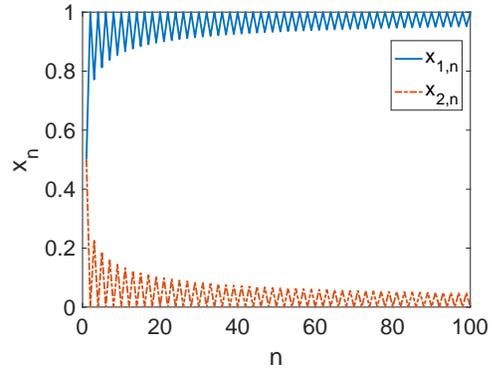}
}
\caption{Performance of the Hedge and Greedy Subgradient algorithms in Example \ref{ex1}.  }\label{fig:ex1_a}
\end{figure}

Importantly, it is easy to verify that common online learning algorithms do not generate sequences with this property, as the following example illustrates.

\begin{example}\label{ex1}
Suppose loss vector $\l_i=(\l_{1,i},\l_{2,i})$ with $\l_{1,i}=(-1)^i$, i.e. sequence $-1$, $+1$, $-1$, $+1$, $\dots$, and $\l_{2,i}=1/(2\sqrt{i})$.   Suppose also we are to choose between $d=2$ fixed actions $z_{1,i}=e_1=(1,0)$ and $z_{2,i}=e_2=(0,1)$, and that $y^*=(1,0)$.    Then $\R_{1,n}=0$ and $\R_{2,n}$ is $\Theta(\sqrt{n})$.   Figures \ref{fig:ex1_a}(a)-(b) show the regret ${\R}_{n}$ when using the Hedge algorithm\footnote{$x_{k,i+1}=w_{k,i}/\sum_{k=1}^d w_{k,i}$ ,$w_{k,i}=e^{-\sum_{j=1}^i\l_{k,j}/\sqrt{i}}$ for $k\in\{1,2\}$.} and Figures \ref{fig:ex1_a}(c)-(d) when using the Greedy Subgradient algorithm\footnote{$x_{i+1}=P_{\S}(x_{i}-\l_{i}/\sqrt{i})$ where $P_{\S}$ denotes the Euclidean projection onto the simplex.}.   Despite the simplicity of the choice to be made in this example it can be seen that the regret ${\R}_{n}$ of both algorithms is $\Theta(\sqrt{n})$, whereas $\min\{\R_{1,n},\R_{2,n}\}=0$.   It can be verified that for both algorithms similar behaviour is observed with constant $\sqrt{n}$ stepsize, and also with the Prod algorithm\footnote{$x_{k,i+1}=w_{k,i}/\sum_{k=1}^d w_{k,i}$, $w_{k,i+1}=w_{k,i}(1-\l_{k,i}/\sqrt{n})$ for $k\in\{1,2\}$.}.
\end{example}

The difficulty here arises because the algorithms do not settle on the best expert $z_{1,i}$, but rather oscillate about a mixture of the actions propsed by the two experts.  Due to the $\Theta(\sqrt{n})$ loss of $z_{2,i}$, such a mixture is liable to have regret $\Theta(\sqrt{n})$ rather than the desired $O(1)$.

\section{Gap Property of the Lazy Subgradient Method}\label{sec:gap}

The lazy subgradient method selects $x_i$ according to,
\begin{align}
x_{i}=P_{\S}\left(-\alpha_i\sum_{j=1}^{i-1}b_j\right)\label{eq:lazy}
\end{align}
for step size $\alpha_i>0$ and $P_{\S}$ is the Euclidean projection onto $d$-simplex $\S$.    Recently, \cite[Lemma 2]{anderson2019optimality} established the following property of the Euclidean projection,
\begin{lemma}[\cite{anderson2019optimality}]\label{lem:one}
Suppose $w\in\mathbb{R}^d$ has two coordinates $k,l$ with $w_k-w_l\ge 1$.  Then $P_{\S}(w)$ has $l$-coordinate zero.
\end{lemma}
\noindent Figure \ref{fig:simplex} illustrates Lemma \ref{lem:one} for $d=2$ dimensions.  Points lying in the region between the two normals are projected onto the interior of the simplex.   All other points are projected onto the closest extreme point, e.g. point $x$ in Figure \ref{fig:simplex}.   Lemma \ref{lem:one} characterises such points.

\begin{figure}
\centering
\includegraphics[width=0.35\columnwidth]{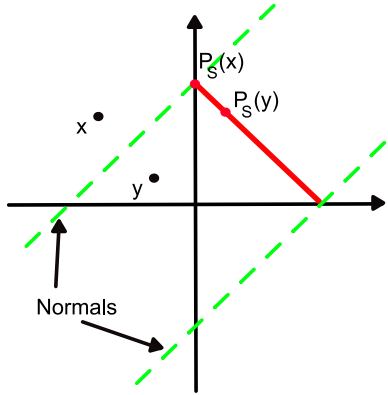}
\caption{Illustrating Lemma \ref{lem:one} on the plane.  The simplex is indicated by the solid line segment, and normals to the two extreme points of the simplex are indicated by the dashed lines.  Points lying above the upper normal or below the lower one are projected onto the corresponding extreme point, e.g. the projection $P_{\S}(x)$ of point $x$ is point (0,1).  Points lying between the normals are projected onto the interior of the simplex, e.g. point $y$. }\label{fig:simplex}
\end{figure}


Applying Lemma \ref{lem:one} to the lazy subgradient method (\ref{eq:lazy}) we immediately have the following result,
\begin{lemma}[Subgradient Gap]\label{lem:two}
Let $k^*\in\arg \min\{\R_{1,n},\dots,\R_{d,n}\}$.   Suppose $\R_{k,n}- \R_{k^*,n}\ge 1/\alpha_n$, $k\in\{1,\dots,d\}\setminus\{k^*\}$ for all $n\ge n_0$ and that {$\|b_i\|_\infty\le 1$.}  That is, the gap between the regret of the best expert $k^*$ and the other experts is at least $1/\alpha_n$.  Then the regret $\R_n$ of the subgradient update (\ref{eq:lazy}) {satisfies $\R_n\le \R_{k^*,n}+ \max\{1,n_0\}$. }
\end{lemma}
\begin{proof}
Begin by observing that
\begin{align}
\R_{k,n}-\R_{k^*,n}=\sum_{i=1}^n \l_i^T \left(z_{k,i}- y^*\right)-\sum_{i=1}^n \l_i^T \left(z_{k^*,i}- y^*\right)=\sum_{i=1}^n \l_i^T \left(z_{k,i}- z_{k^*,i}\right)=\L_{l,n}-\L_{k^*,n}\notag
\end{align}
and so $\R_{k,n}- \R_{k^*,n}\ge 1/\alpha_n$ implies $\L_{k,n}- \L_{k^*,n}\ge 1/\alpha_n$, where $\L_{k,n}=\sum_{i=1}^n\l_i^Tz_{k,i}=\sum_{i=1}^n b_i^Te_k$, $k=1,\dots,d$ is the cumulative loss incurred by the $k$'th expert $z_{k,i}$.   Without loss of generality let $k^*=1$ since we can always permute the experts so that this holds.  Observe that $\L_{k,n}\ge \L_{1,n}+1/\alpha_n$ implies $\sum_{i=1}^n b_i (e_k - e_1) =-\sum_{i=1}^n b_i (e_1 - e_k)\ge1/\alpha_n$.   Letting $w$ be the vector $w=-\alpha_n\sum_{i=1}^n b_i^T$, then $w_1-w_k = -\alpha_n\sum_{i=1}^nb_i(e_1-e_k) \ge 1 $.  By Lemma \ref{lem:one} it follows that $P_{\S}(w)$ has $k$ coordinate zero.  Since by assumption this holds for all $k\ge 2$ then only the first coordinate of $P_{\S}(w)$ is non-zero for $n\ge n_0$ i.e. action $z_{1,i}$ is applied for $n\ge \max\{1,n_0\}$, where we need to take the max of $n_0$ and 1 since projection $P_{\S}$ is only used to select $x_i$ from step $i=2$ onwards and the initial $x_1$ is arbitrary.  The regret $\R_n=\sum_{i=1}^{n} \l_i^T \left(z_{1,i}- y^*\right)+\sum_{i=1}^{n_0} \l_i^T \left(y_i-z_{1,i}\right)=\R_{1,n}+\sum_{i=1}^{\max\{1,n_0\}} b_i^T \left(x_i-e_1\right)$.  Since $x_i$, $e_i$ lie in the simplex the last term is upper bounded by $\max\{1,n_0\}$.
\end{proof}
\noindent Note that we can easily tighten up this bound to replace {the $\max\{1,n_0\}$ term with an $O(\sqrt{n_0})$ }one via the usual worst-case bound on the regret of the subgradient method over the first $n_0$ steps.   

Revisiting Example \ref{ex1} in light of Lemma \ref{lem:two}, it can be verified that $\R_{2,n}- \R_{1,n}\ge 0.5\sqrt{n}$ and so Lemma \ref{lem:two} holds with $n_0=0$ and $\alpha_n=2/\sqrt{n}$.  Hence, subgradient update $(\ref{eq:lazy})$ with step size $\alpha_n=2/\sqrt{n}$ yields regret {$\R_n \le \R_{1,n}+1$ i.e.} regret of the same order as the regret of the best expert, as desired.    See Figure \ref{fig:ex1_fixed}.

\begin{figure}
\centering
\subfigure[]{
\includegraphics[width=0.45\columnwidth]{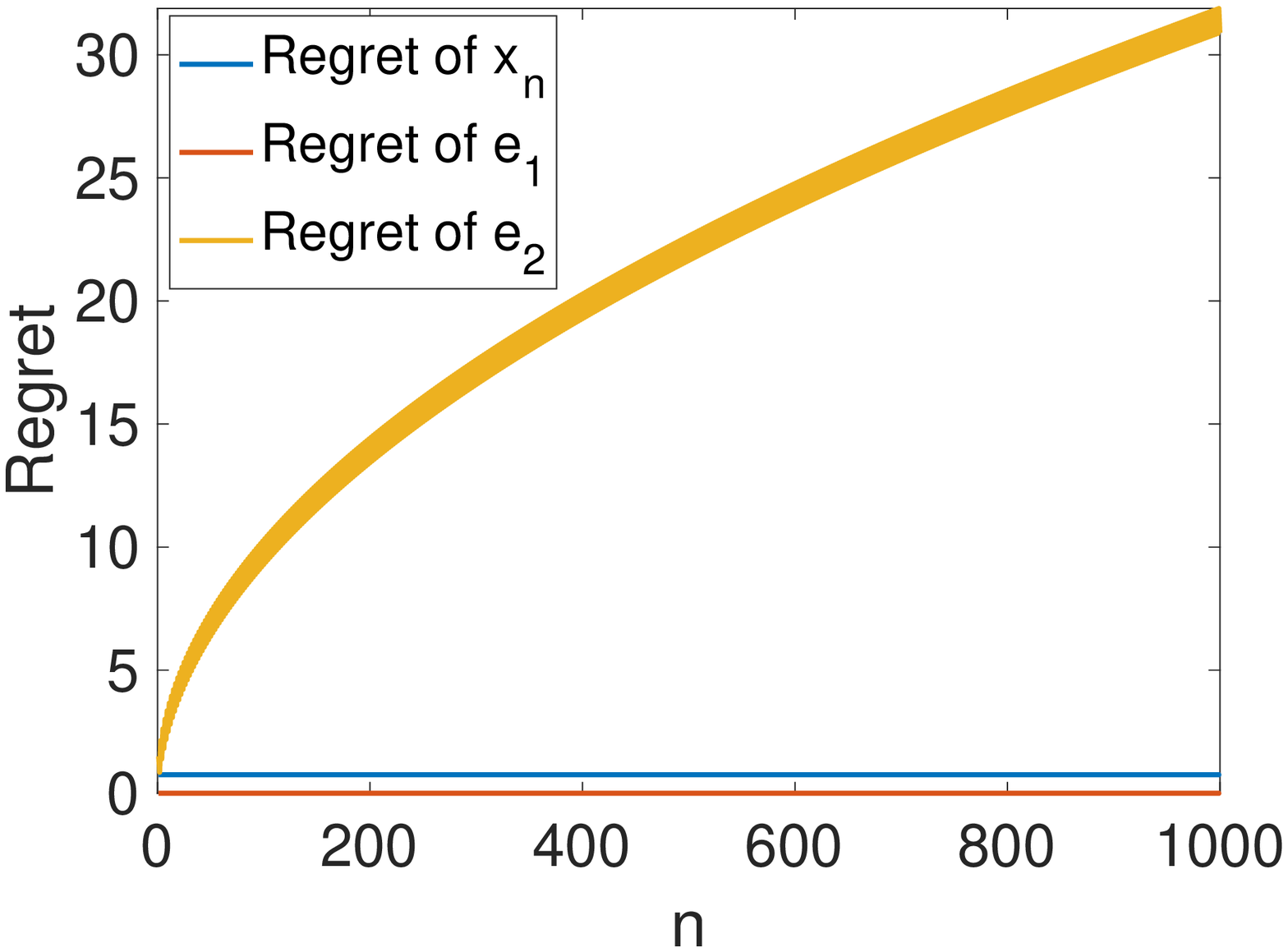}
}
\subfigure[]{
\includegraphics[width=0.45\columnwidth]{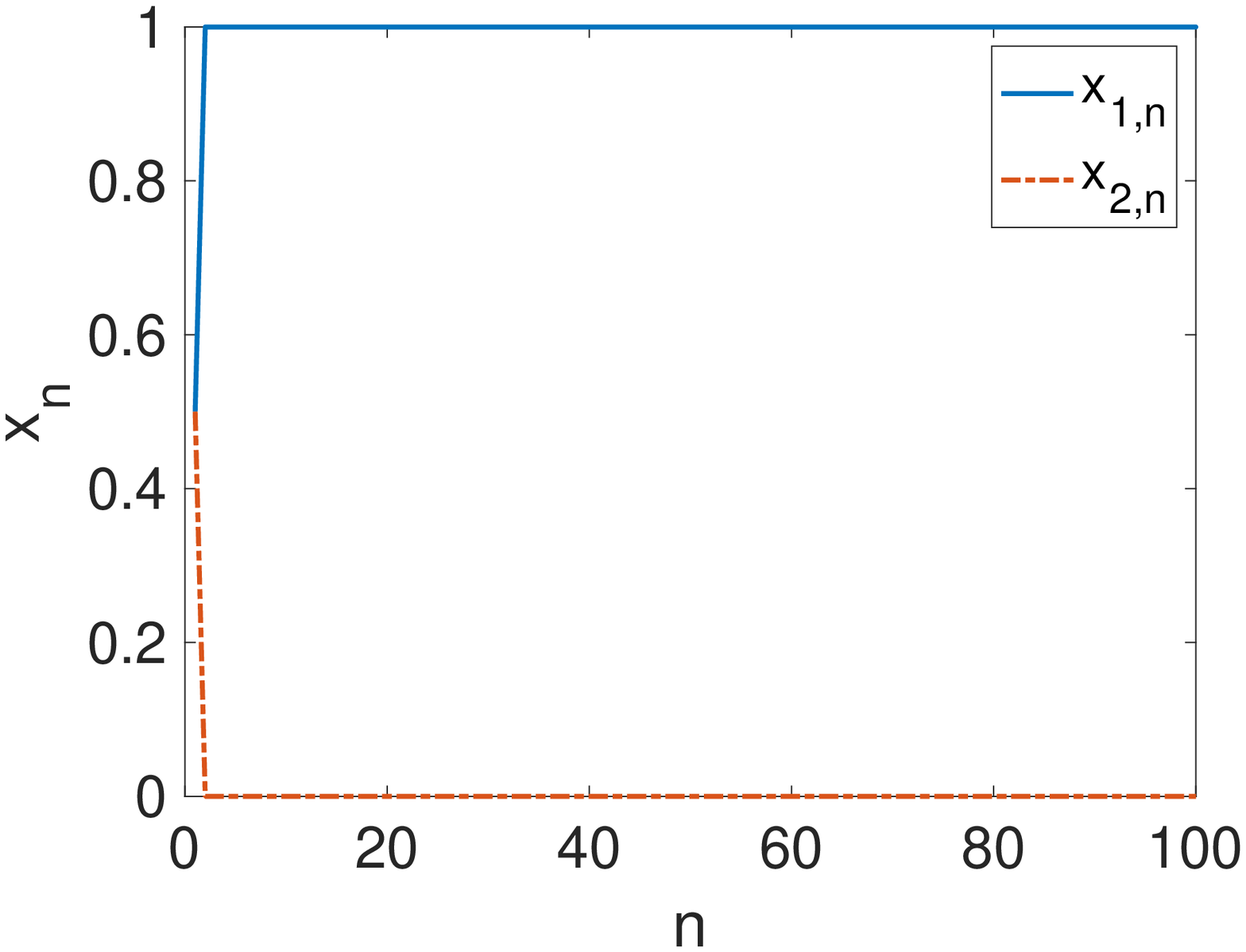}
}
\caption{Performance of the Lazy Subgradient algorithm in Example \ref{ex1} (with step size $\alpha_i=2/\sqrt{i}$). }\label{fig:ex1_fixed}
\end{figure}

More generally, Lemma \ref{lem:two} defines a class of ``easy'' cases where the regret of the best expert is sufficiently distinct from the other experts in the sense that they differ by at least $1/\alpha_n$.   For these easy cases the lazy subgradient method achieves efficient regret.  Typically we need to choose the step size $\alpha_n$ proportional to $1/\sqrt{n}$ in order to ensure good worst case performance in which case we need the gap between regrets to be proportional to $1/\sqrt{n}$ in order to apply Lemma \ref{lem:two}.

Another ``easy'' case where we might reasonably expect a learning algorithm to have efficient regret is when  all of the experts have similar regret.    Unfortunately it is not hard to devise examples where the subgradient method (\ref{eq:lazy}) yields $\Theta(\sqrt{n})$ regret even though the regrets of the individual experts are all $O(1)$, as the following example illustrates.
\begin{example}\label{ex2}
Suppose $\l_{1,i}=(-1)^{i+1}$, i.e. sequence $+1,-1,+1,-1,+1,\dots$, and $\l_{2,i}=(-1)^i$, i.e. sequence $-1,+1,-1,+1,-1,\dots$.   Suppose $d=2$, $z_{1,i}=(1,0)$ and $z_{2,i}=(0,1)$ and that $y^*=(0,1)$.   Since $-1\le\sum_{i=1}^n \l_{k,i} \le 1$ for $k=1,2$ the regret of both experts is $O(1)$.  Figure \ref{fig:ex2}(a) shows the regret when these experts are combined using the subgradient method.  It can be seen that the regret grows as $\Theta(\sqrt{n})$.    Figure \ref{fig:ex2}(b) plots $x_n$ vs time.  It can be seen that the action oscillates about the $(0.5,0.5)$ point.   The difficulty arises because the sign differences between $\l_{1,i}$ and $\l_{2,i}$ mean that such oscillations can yield larger cumulative loss than any fixed combination of $\l_{1,i}$ and $\l_{2,i}$.   
\end{example}
%
%
\begin{figure}
\centering
\subfigure[]{
\includegraphics[width=0.45\columnwidth]{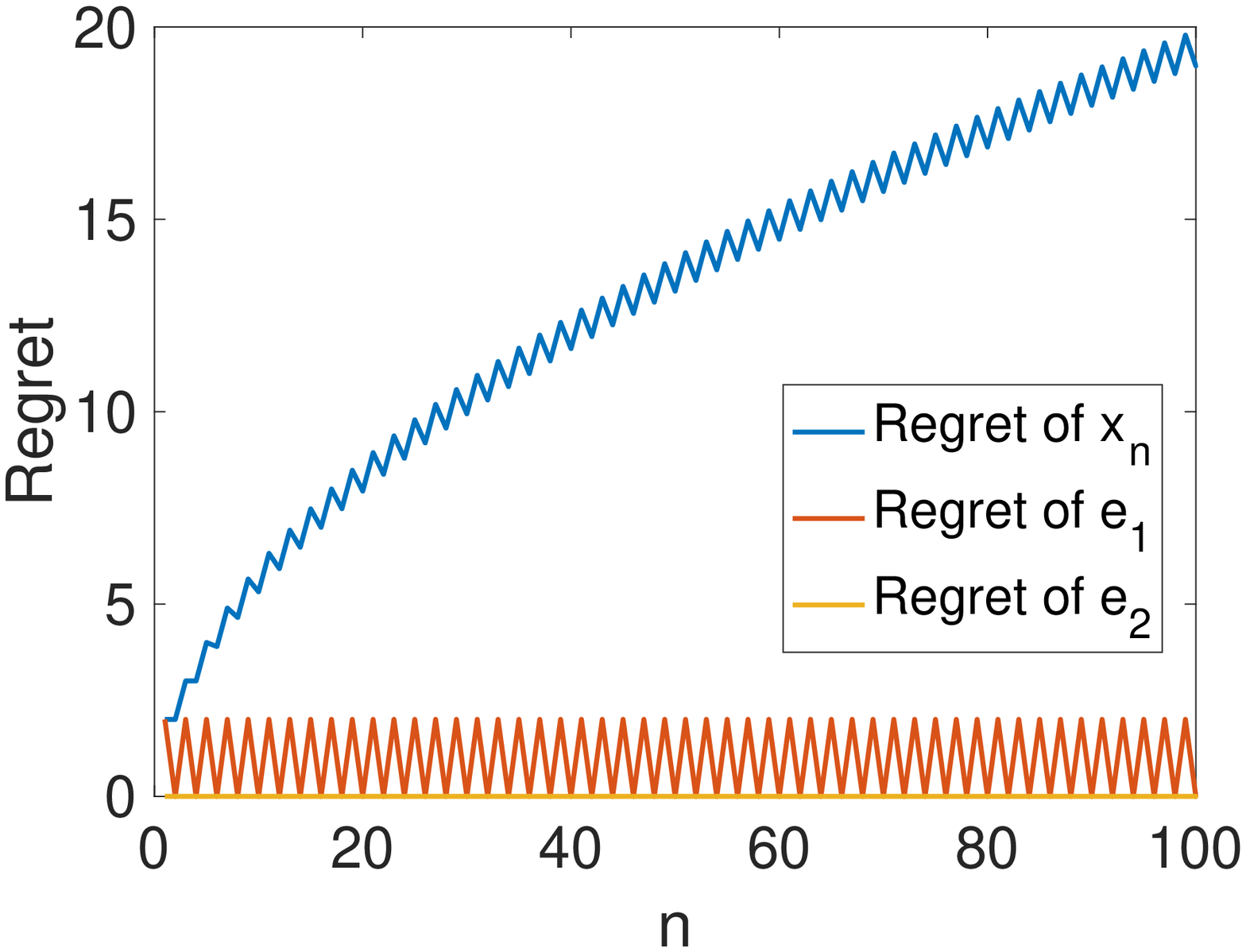}
}
\subfigure[]{
\includegraphics[width=0.45\columnwidth]{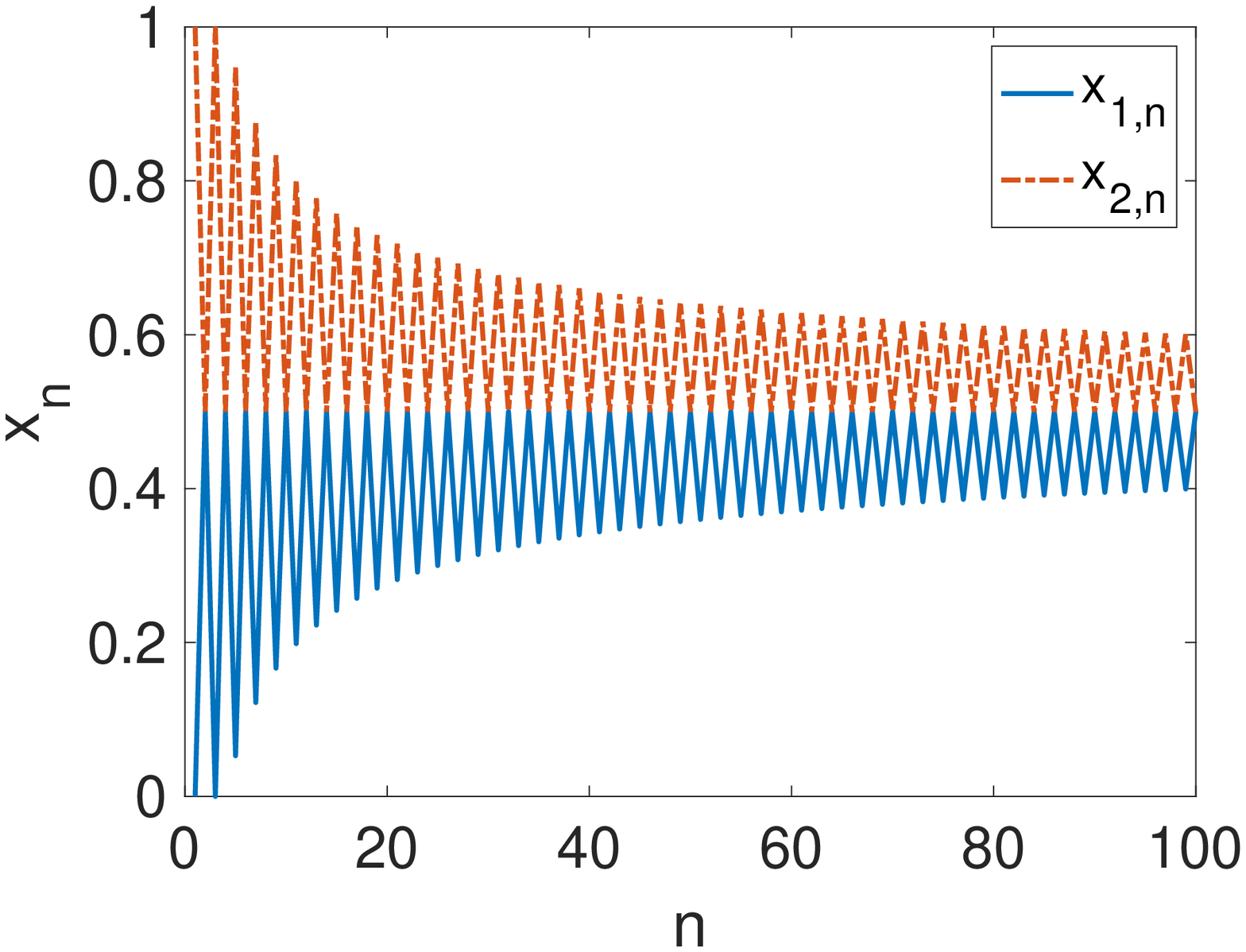}
}
\caption{Example \ref{ex2} where individual experts have regret upper bounded by a constant but when combined using subgradient method the resulting actions have $\Theta(\sqrt{n})$ regret.  Left hand plot shows the regret of the combined action $x_n$ taken by the subgradient method and also the regret of experts 1 and 2 (regret shown is with respect to expert 2 but it also $O(1)$ wrt algorithm 2, or any fixed combination of the two).  Right hand plots action $x_n$ taken by subgradient vs time.}\label{fig:ex2}
\end{figure}

\section{Biased Lazy Subgradient Method}\label{sec:bias}
\subsection{Learning the Best of Two Experts}
It turns out that it is indeed possible to use the Lazy Subgradient method to achieve efficient regret both when the gap condition in Lemma \ref{lem:two} holds and when the difference between the regrets of the available experts is small.   However, this requires biasing the loss sequence to which the subgradient method is applied.   We begin by considering the case of $d=2$ experts and at step $i$ selecting,
\begin{align}
x_i=P_{\S}(-\alpha_{i-1}(A_{i-1},B_{i-1})^T)=P_{\S}(-\alpha_{i-1}\sum_{j=1}^{i-1}(a_j,b_j)^T) \label{eq:alg2a_subgrad}
\end{align}
where $A_i,B_i\in\mathbb{R}$, $a_i=A_i-A_{i-1}$ with $a_1=A_1$, $b_i=B_i-B_{i-1}$ with $b_1=B_1$.  From now on we fix parameter $\alpha_i=\sfrac{1}{\sqrt{i}}$.  Observe that this is just the Lazy Subgradient update applied to the sequence of vectors $(a_i,b_i)$, $i=1,2,\dots$.  We have in mind selecting $B_i=\R_{2,i}-\R_{1,i}=\sum_{j=1}^{i}\l_j^T(z_{2,j}-z_{1,j})$ and using $A_i$ as benchmark against which to compare $B_i$.  

We can rewrite this update equivalently as,
\begin{align}
x_{1,i}=P_\I(\tilde{A}_i+\frac{1}{\sqrt{i-1}}B_{i-1}),\ x_{2,i}=1-x_{1,i} \label{eq:alg2a}
\end{align}
where $\I$ is the interval $[0,1]$ and bias $\tilde{A}_i=\sfrac{1}{2}-{A_{i-1}}/{\sqrt{i-1}}$.  To see this observe that $P_\S(w)=\arg\min_{x\in\S} \|w-x\|_2$ with $\S=\{(x_1,x_2): x_1+x_2=1, x_1,x_2\ge 0\}=\{(x_1,x_2): x_1=(\sfrac{1}{2}+\tilde{q}), x_2=(\sfrac{1}{2}-\tilde{q}), \tilde{q}\in[-\sfrac{1}{2},\sfrac{1}{2}]\}$.  Hence, $$P_\S(w)=\arg\min_{\tilde{q}\in[-\frac{1}{2},\frac{1}{2}]} \sqrt{(w_1-(\frac{1}{2}+\tilde{q}))^2+(w_2-(\frac{1}{2}-\tilde{q}))^2}=\arg\min_{\tilde{q}\in[-\frac{1}{2},\frac{1}{2}]}|w_1-w_2-\tilde{q}|$$ (expanding the square and dropping constant terms).   Changing variables to $x_1=(\sfrac{1}{2}+\tilde{q})$ now yields (\ref{eq:alg2a}).

When written in the form (\ref{eq:alg2a}) it can be seen that when $B_{i}/\sqrt{i}\le-\tilde{A}_i$ then $x_{1,i}=0$ and when $B_{i}/\sqrt{i}\ge1-\tilde{A}_i$ then $x_{1,i}=1$.   Hence, when $B_i=\R_{2,i}-\R_{1,i}$ then $x_{1,i}=0$ (thus $x_{2,i}=1$) when $\R_{2,i}\le\R_{1,i}-\sqrt{i}\tilde{A}_i$ and $x_{1,i}=1$ when $\R_{2,i}\ge\R_{1,i}+\sqrt{i}(1-\tilde{A}_i)$.  That is, we retain a gap property similar to that discussed in Section \ref{sec:gap}, with the gap now tunable by adjusting $\tilde{A}_i$.  Hence, update (\ref{eq:alg2a}) continues to achieve efficient regret in the easy case where there is a large gap between the regrets of the available experts.

Secondly, when $B_i$ is less than $\Theta(\sqrt{i})$ we have that $-\alpha_iB_i$ converges to the origin and $x_{1,i}=P_\I(\tilde{A}_i)$.   Hence, when $B_i=\R_{2,i}-\R_{1,i}$ then we can use $\tilde{A}_i$ to control the action taken when the difference in the regrets of the two experts is small.   In particular, when $\tilde{A}_i>1$ then $|\alpha_i B_n|\le \tilde{A}_i-1$ ensures $x_{1,i}=1$ i.e. we default to use of expert 1 when the difference in regrets is small.   Hence, unlike the original lazy subgradient update in Section \ref{sec:gap} the biased update (\ref{eq:alg2a}) also achieves efficient regret in the second easy case where the available experts have similar regrets.

We formalise these observations in the following lemma,
\begin{lemma}[Equilibrium Points]\label{lem:equil} 
Under update (\ref{eq:alg2a}), when either  $B_n \ge A_n+\sqrt{n}/2$ or  $|B_n|\le -(A_n+\sqrt{n}/2)$ for all $n\ge n_0$ then $x_{i}=(1,0)$ for all $n\ge n_0$.  When $B_n\le A_n-\sqrt{n}/2$ for all $n\ge n_0$ then $x_{i}=(0,1)$ for all $n\ge n_0$.
\end{lemma}

We now establish the worst-case performance of update (\ref{eq:alg2a}).


{
\begin{lemma}[FTL]\label{lem:ftrl} 
Under update (\ref{eq:alg2a}) we have for each $w\in [0,1]$ the inequality $$\sum_{i=1}^{n}b_i(x_{2,i}-w)\le  3\sqrt{n}   + 2\sum_{i=1}^{n} |a_{i}|$$ 
\end{lemma}
\begin{proof} Let $R_i(x)=\frac{\sqrt{i}}{2}\|x\|^2$.  By Lemma \ref{lem:ftrl} we have   $\sum_{i=1}^n (a_i,b_i)^T(x_i-x^*) \le R_n(x^*)+\sum_{i=1}^{n}(a_i,b_i)^T(x_i-x_{i+1})$ for each $ x^* \in \S$.  For the sum on the right we have
	\begin{align}\sum_{i=1}^{n}(a_i,b_i)^T(x_i-x_{i+1})& \le \sum_{i=1}^{n}\|(a_i,b_i)\|\|x_i-x_{i+1}\| \label{ftl}  \\
	& \le  \sum_{i=1}^{n} |b_i| \|x_i-x_{i+1}\| +  \sum_{i=1}^{n} |a_i| \|x_i-x_{i+1}\|  \le  \sum_{i=1}^{n}  \|x_i-x_{i+1}\| +  \sum_{i=1}^{n} |a_i| \notag
	\end{align} 
	
where the last line inequality uses the assumption $\|b_i\| \le 1$.  By Lemma \ref{lem:strong} we have $\|x_i-x_{i+1}\| \le \frac{1+|a_{i+1}| }{\sqrt i} + \frac{1}{4   i}$. hence right-hand-side is at most $$\sum_{i=1}^{n} \left( \frac{1}{\sqrt i} + \frac{1}{4i}+\frac{|a_{i+1}|}{\sqrt i} \right) \le 2 \sqrt n + \frac{\log n}{4} +\sum_{i=1}^{n}\frac{|a_{i+1}|}{\sqrt i}.$$

Combining the above wie
By the above (\ref{ftl}) gives  
\begin{align*}
\sum_{i=1}^n b_i(x_{2,i}-x^*_2) 
\le \frac{\sqrt{n}}{2}\|x^*\|^2+2 \sqrt n + \frac{\log n}{4} +\sum_{i=1}^{n}\frac{|a_{i+1}|}{\sqrt i}-\sum_{i=1}^n a_i(x_{1,i}-x^*_1) \\ \le \frac{5 }{2}\sqrt{n}   + \frac{\log n}{4}+2\sum_{i=1}^n |a_i| \le 3\sqrt{n}    +2\sum_{i=1}^n |a_i|
\end{align*}
where the first inequality follows from how $|x_{1,i}-x^*_1| \le 1$.  Since the above holds for all $x^* \in \S$ it holds for $x^* = (w,1-w)$.
\end{proof} }

{ 
\begin{lemma}[Worst-case regret]\label{lem:worst}
Under update (\ref{eq:alg2a}) with $B_i=\R_{2,i}-\R_{1,i}$ then regret $\R_n\le  \min\{\R_{1,n},\R_{2,n}\}+3\sqrt{n} + 3\sum_{i=1}^n |a_i|$.
\end{lemma}
\begin{proof} {
Begin by observing that for any $w \in [0,1]$ we have 
\begin{align*}
\R_n&=\sum_{i=1}^n \l_{i}^T(z_{1,i}x_{1,i}+z_{2,i}x_{2,i}-y^*)=\sum_{i=1}^n \l_{i}^T(z_{1,i}(1-x_{2,i})+z_{2,i}x_{2,i}-y^*)\\
&=\sum_{i=1}^n \l_{i}^T(z_{1,i}-y^*)+\l_i^T(z_{2,i}-z_{1,i})x_{2,i}\\
&=\sum_{i=1}^n \l_{i}^T( (z_{1,i}-y^*)(1-w)+(z_{2,i}-y^*)w )+\l_i^T(z_{2,i}-z_{1,i})(x_{2,i}-w)\\
&=(1-w)\R_{1,n}+w\R_{2,n}+\sum_{i=1}^n\l_i^T(z_{2,i}-z_{1,i})(x_{2,i}-w)
\end{align*}  
The previous lemma says the sum is  at most $3\sqrt{n}+2\sum_{i=1}^n |a_i|$.  
For the first part write $F =  (1-w)\R_{1,n}+w\R_{2,n}$. To show $ F(w) \le \min\{\R_{1,n},\R_{2,n}\} + \sum_{i=1}^n |a_i|$ consider two cases.  Case (i): $B_n > A_n$. Then $B_n =\R_{2,n}-\R_{1,n} > A_n$ and so $\R_{1,n} < \R_{2,n} + |A_n|$.  Hence for $w=0$ we have $F= \R_{1,n} \le \R_{1,n} + |A_n|$ and $F= \R_{1,n} < \R_{2,n} + |A_n|$. Combining the two we have $F\le \min\{\R_{1,n},\R_{2,n}\} + |A_n| \le  \min\{\R_{1,n},\R_{2,n}\} + \sum_{i=1}^n |a_i|$.
Case (ii): $B_n \le A_n$. We have $\R_{2,n} \le \R_{1,n} + A_n \le \R_{1,n} + |A_n|$.  Choosing  $w=1$ the rest of the proof is similar. 
}
\end{proof}}

Combining the above lemmas yields the following,
{

\begin{theorem}[Biased Subgradient Efficiency]\label{lem:first}
Using update (\ref{eq:alg2a}) with $B_i=\R_{2,i}-\R_{1,i}=\sum_{j=1}^{i}\l_j^T(z_{2,j}-z_{1,j})$ and $A_i=-(\frac{\sqrt{i}}{2}+\beta\log i)$, $\beta\ge 0$ we have 
\begin{enumerate}
\item \emph{Distinct Experts}. When $\R_{2,n}-\R_{1,n} \ge 0$ or $\R_{2,n}-\R_{1,n} \le -\sqrt{n}-\beta\log n$ for all $n\ge n_0$ then $\R_n\le \min\{\R_{1,n},\R_{2,n}\}+M(n_0)$.
\item \emph{Similar Experts}. When $|\R_{2,n}-\R_{1,n}| \le  \beta \log n$ for all $n\ge n_0$ then $\R_n\le \R_{1,n}+M(n_0)$.
\item \emph{Worst Case}. Otherwise $\R_n \le  \min\{\R_{1,n},\R_{2,n}\}+\frac{9}{2}\sqrt{n} + 3 \beta \log n$.
\end{enumerate}
where $M(n_0):=\frac{9}{3}\sqrt{n_0}+3 \beta \log n_0$.
\end{theorem}
\begin{proof}
	For the worst case we use Lemma \ref{lem:worst}. Observe $a_1 = A_1 = - \frac{1}{2}$ and for $i>1$ we have \begin{align*}&a_i=A_i-A_{i-1}= \frac{\sqrt {i-1} - \sqrt i}{2} + \beta \log(i-1 )-\beta \log(i) =   \frac{\sqrt {i-1} - \sqrt i}{2} + \beta \log \left (\frac{i-1}{i} \right )  
	\end{align*} 
	\begin{align*}|a_i|\le \frac{\sqrt {i } - \sqrt {i-1}}{2}   + \beta  \log \left ( \frac{ i}{i-1}\right )  \le \frac{\sqrt {i } - \sqrt {i-1}}{2}   + \beta  \log \left (1+ \frac{ 1}{i-1}\right )   \le \frac{\sqrt {i } - \sqrt {i-1}}{2}   + \frac{\beta}{i}
	\end{align*}    Hence $\sum_{i=1}^n |a_i|\le \frac{\sqrt n}{2} + \beta \log n$ and the worst case now follows from Lemma \ref{lem:worst}. The ``distinct'' and ``similar'' expert cases now follow from application of Lemma \ref{lem:equil} and noting that by Lemma \ref{lem:worst} the regret over the first $n_0$ steps is at most $M(n_0)$.  
\end{proof}}


\begin{figure}
\centering
\subfigure[]{
\includegraphics[width=0.45\columnwidth]{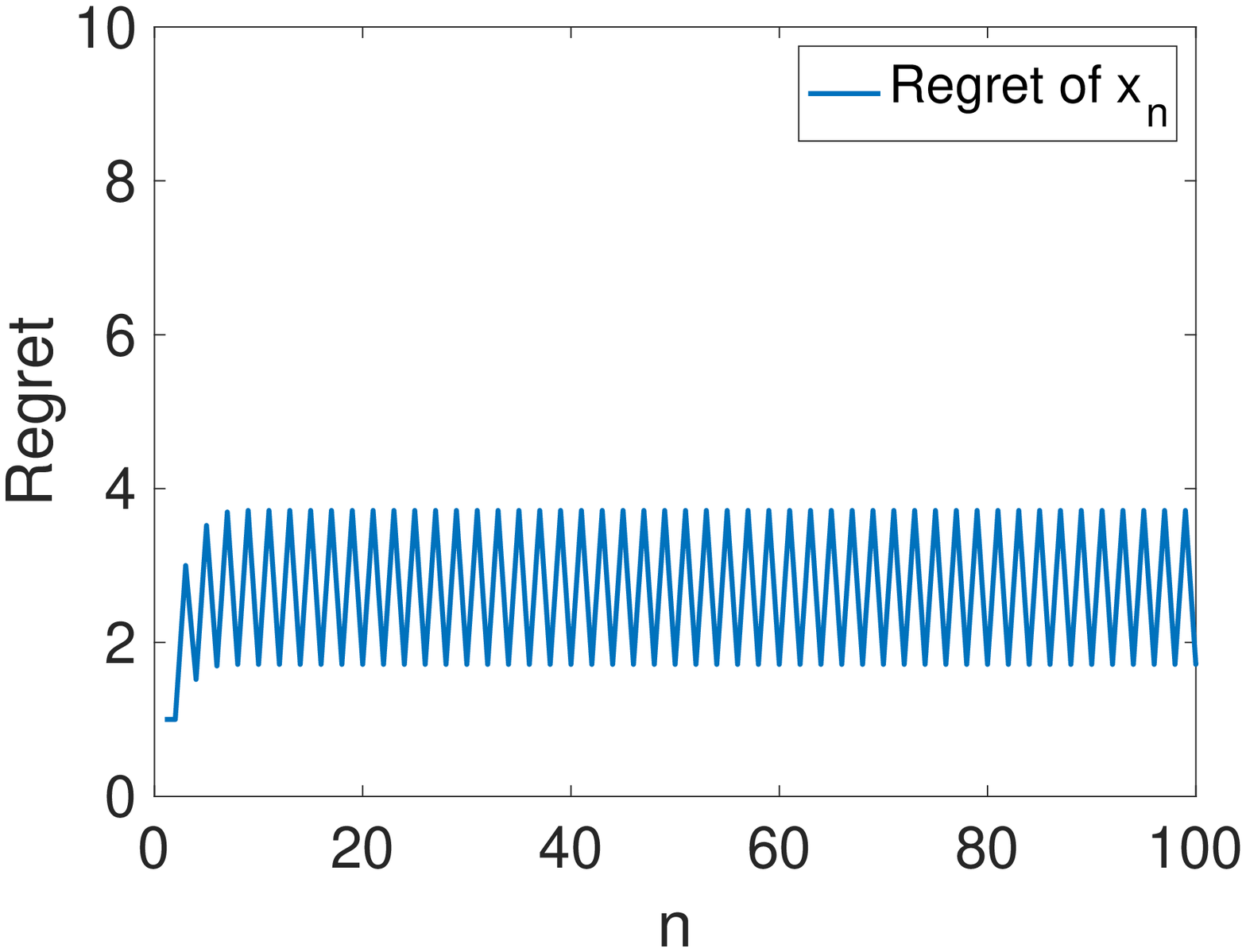}
}
\subfigure[]{
\includegraphics[width=0.45\columnwidth]{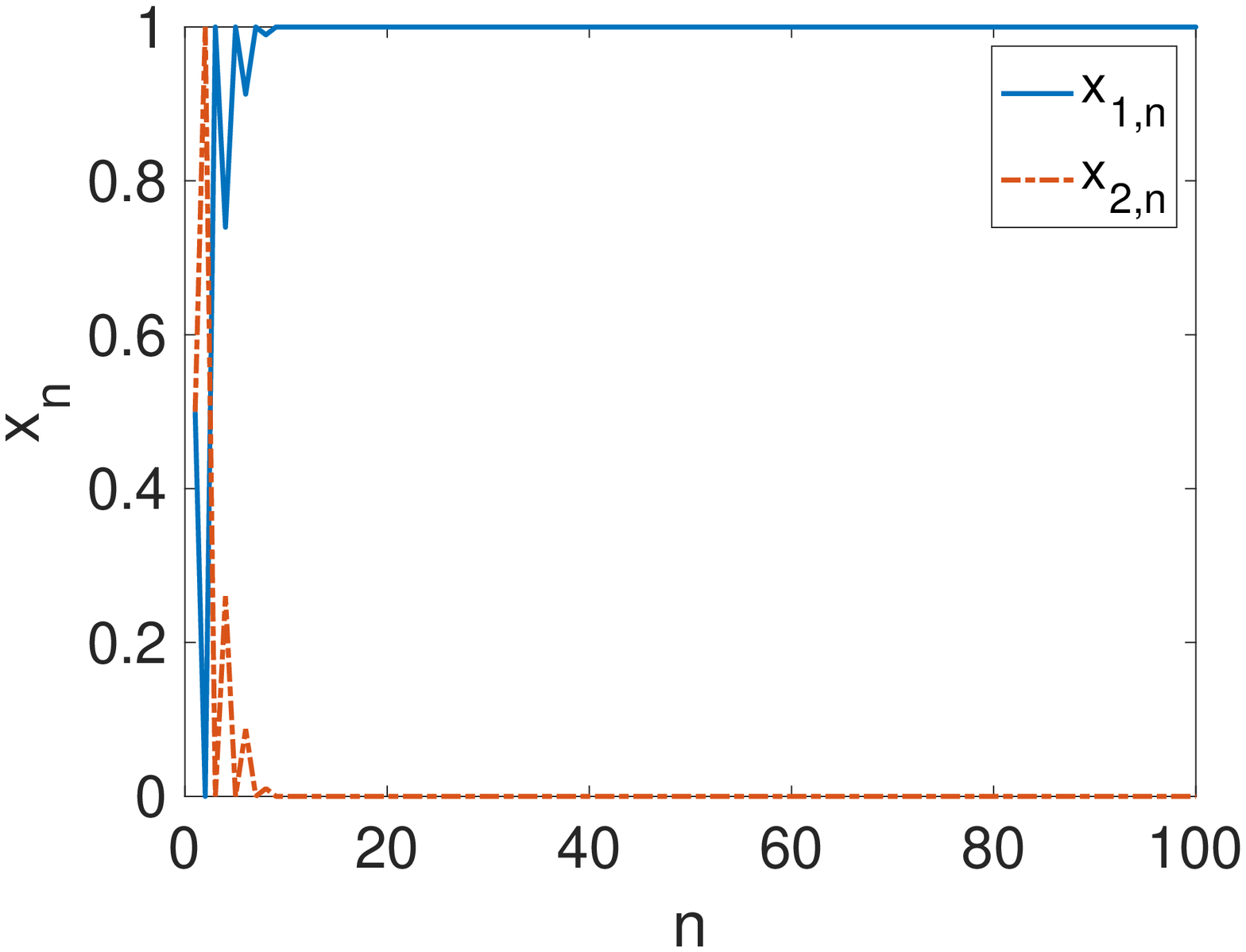}
}
\caption{Example \ref{ex2} where experts are now combined using the biased lazy subgradient method (\ref{eq:alg2a}) with $A_i=-(\frac{\sqrt{i}}{2}+\log i)$.  Left hand plot shows the regret of the combined action $x_n$ with respect to expert 2 and the right hand plot shows the action $x_n$ taken vs time.}\label{fig:ex2_subgradv2}
\end{figure}

Revisiting Example \ref{ex2} using the Biased Lazy Subgradient method (\ref{eq:alg2a}), Figure \ref{fig:ex2_subgradv2} plots the performance.  This can be compared directly with Figure \ref{fig:ex2}.  It can be seen that, in line with Theorem \ref{lem:first}, the Biased Lazy Subgradient method settles quickly on expert 1 and achieves $O(1)$ regret in contrast to the $\Theta(\sqrt{n})$ regret achieved by the Lazy Subgradient method.

\subsection{Discussion}

When combining experts with $\Theta(\sqrt{n})$ regret Theorem \ref{lem:first} says that the combined regret will remain $\Theta(\sqrt{n})$.  When combining experts where one has $\Theta(\sqrt{n})$ regret and the other has regret less than this, e.g. $\Theta(\log n)$ or $\Theta(1)$ then the combined regret will be the same order as the better expert.  When combining experts with regret less than $\beta\log{n}$ then the combined regret will remain less than $\beta\log{n}$, and when combining experts with $\Theta(1)$ regret then the combined regret will remain $O(1)$.   Probably the main limitation highlighted by Theorem \ref{lem:first} is that when one expert has $\Theta(\log n)$ regret and the other $\Theta(1)$ regret then Theorem \ref{lem:first} says that the combined expert may have $\Theta(\log n)$ regret.  This behaviour can actually happen, as illustrated by the following example. 

\begin{figure}
\centering
\subfigure[]{
\includegraphics[width=0.45\columnwidth]{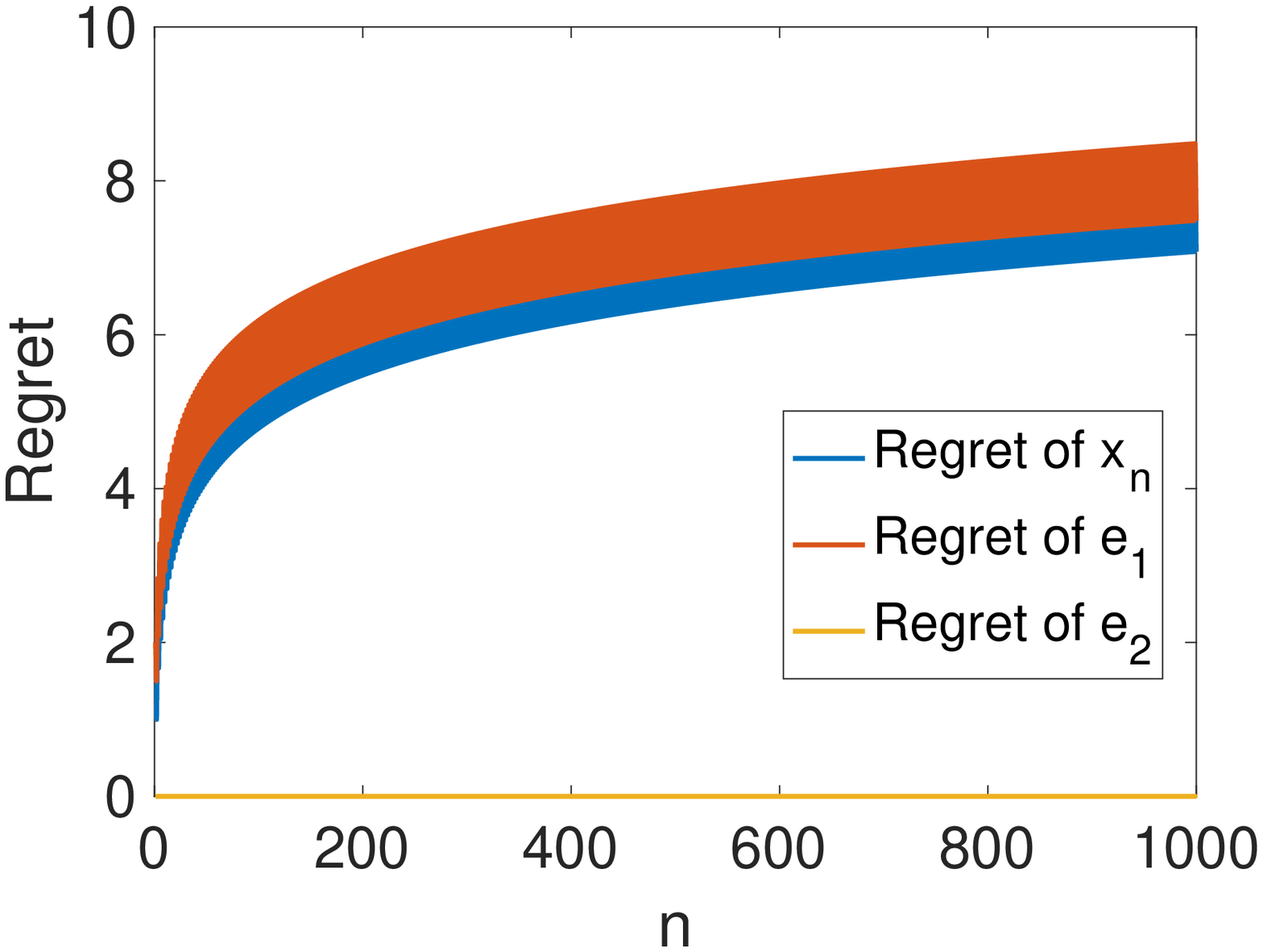}
}
\subfigure[]{
\includegraphics[width=0.45\columnwidth]{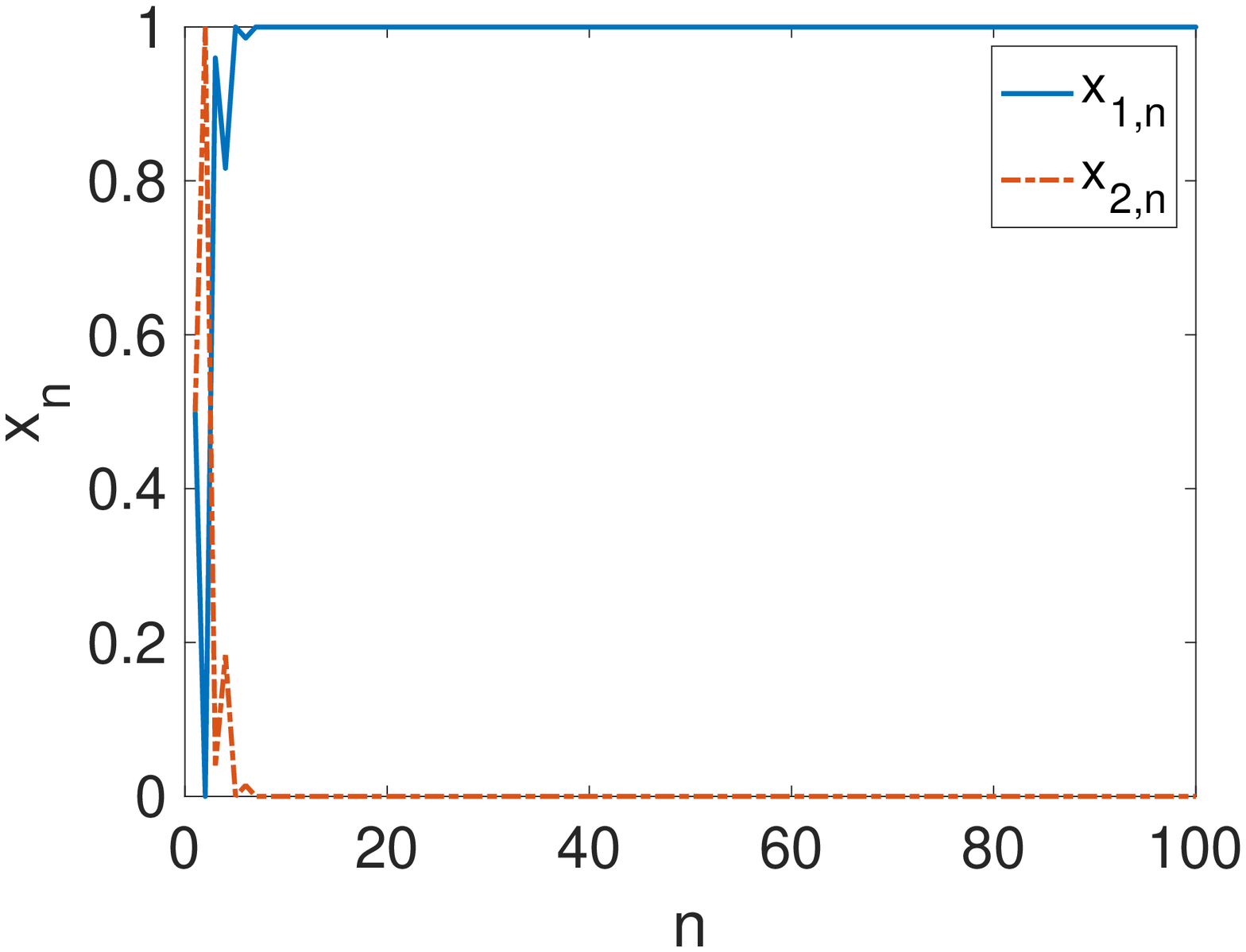}
}
\caption{Example \ref{ex4} where combining experts with $\Theta(\log n)$ and $\Theta(1)$ regret using the biased subgradient method yields $\Theta(\log{n})$ regret.  Left hand plot shows the regret of the combined action $x_n$ and also the regret of experts 1 and 2 with respect to expert 2.  Right hand plots action $x_n$ taken vs time.}\label{fig:ex4}
\end{figure}

\begin{example}\label{ex4}
Suppose $\l_{1,i}=1/i$ and $\l_{2,i}=(-1)^{i}$, i.e. sequence $-1,+1,-1,+1,\dots$.   Suppose $d=2$ and $z_{1,i}=(1,0)$ for all $i=1,2,\dots$, $z_{2,i}=(0,1)$ and $y^*=(0,1)$.   The regret of the first expert is $\Theta(\log n)$.  Figure \ref{fig:ex4}(a) shows the regret when these experts are combined using biased subgradient method (\ref{eq:alg2a}) with $A_i=-(\sqrt{i}+2\log i)$.  It can be seen that the regret grows as $\Theta(\log{n})$.    Figure \ref{fig:ex4}(b) plots $x_n$ vs time.  
\end{example}

\subsection{Combining Two Learning Algorithms}

Theorem \ref{lem:first} applies to general loss sequences and requires a gap of $\Theta(\sqrt{n})$ between the regrets of the two experts in order for the biased subgradient algorithm to achieve efficient regret.   A natural question is whether there exists classes of loss for which we can significantly shrink, or even remove, this gap.   With this in mind, one class of particular interest is where the experts $z_{1,n}$ and $z_{2,n}$ are generated by learning algorithms converging at different rates to the same optimum.   In this case we expect  $|\l_n^T(z_{2,n}-z_{1,n})|$ to be at most $O(1/\sqrt{n})$ and we exploit this to distinguish between experts with regrets that differ by $\Theta(\log n)$ rather than by $\Theta(\sqrt{n})$.   

The source of the $\Theta(\sqrt{n})$ gap requirement in Theorem \ref{lem:first} is that $\alpha_i$ must be $\Theta(1/\sqrt{n})$ in order to ensure $O(\sqrt{n})$ worst-case regret but consequently $-\alpha_i B_i=-(\R_{2,i}-\R_{1,i})/\sqrt{i}$ converges to the origin when $R_{2,i}-\R_{1,i}$  is less than $O(\sqrt{n})$.  As a result, in this case update (\ref{eq:alg2a}) cannot distinguish between the experts.   But when we know in advance that $\R_{2,i}-\R_{1,i}$ grows by no more than $O(\sqrt{n})$ then we can rescale $B_i$ so that $-B_i/\sqrt{i}$ differs between low regret experts.  Of course any such rescaling must maintain the growth of $B_i$ at no more than $O(n)$ in order to retain the worst case performance guarantee.  We have the following,


\begin{theorem}\label{lem:second}{
Suppose all $|\l_n^T(z_{2,n}-z_{1,n})| \le \lambda/(2\sqrt{n})$ for some $\lambda\ge 0$. Using update (\ref{eq:alg2a}) with $B_i=\sqrt{i}(\R_{2,i}-\R_{1,i})$ and $A_i=-\sqrt{i}(1+\beta \log i)$, $\beta\ge 0$ we have}
\begin{enumerate}
\item \emph{Distinct Experts}. When $\R_{2,n}-\R_{1,n} \ge 0$ or $\R_{2,n}-\R_{1,n} \le -1-\beta \log n$ for all $n\ge n_0$ then $\R_n\le \min\{\R_{1,n},\R_{2,n}\}+M(n_0)$.
\item \emph{Similar Experts}. When $|\R_{2,n}-\R_{1,n}| \le  \beta \log n$ for all $n\ge n_0$ then $\R_n\le \R_{1,n}+M(n_0)$.
\item \emph{Worst Case}. { Otherwise $\R_n \le  \min\{\R_{1,n},\R_{2,n}\}+1 + \beta \log n +\lambda\sqrt{n}$.}
\end{enumerate}
where {$M(n_0):=1+\beta \log n_0+\lambda\sqrt{n_0}$}.  
\end{theorem}
\begin{proof}
We begin with the worst case.  From the proof of Lemma \ref{lem:worst} we have for all $w \in [0,1]$ the inequality \begin{align}\label{Daron1}
\R_n=(1-w)\R_{1,n}+w\R_{2,n}+\sum_{i=1}^n\l_i^T(z_{2,i}-z_{1,i})(x_{2,i}-w). \end{align}  
The second sum is at most 
\begin{align*}
 \sum_{i=1}^n| \l_i^T(z_{2,i}-z_{1,i}) | |x_{2,i}-w| \le  \sum_{i=1}^n| \l_i^T(z_{2,i}-z_{1,i}) |  \le \sum_{i=1}^n \frac{\lambda}{2 \sqrt i} \le \lambda \sqrt n. 
\end{align*}
For the first part of (\ref{Daron1}) write $F =  (1-w)\R_{1,n}+w\R_{2,n}$ and consider two cases. {Case (i): $B_n>  A_n$. We have $\R_{2,n}-\R_{1,n} \ge -\left( 1 + \beta \log n\right)$ and $\R_{1,n} \le \R_{2,n} + \left( 1 + \beta \log n\right)$. Hence for $w = 0$ we have $F = \R_{1,n}\le \R_{2,n} + \left( 1 + \beta \log n\right)$. Clearly we have $F = \R_{1,n}\le \R_{1,n} + \left( 1 + \beta \log n\right)$ and so $F \le \min\{ \R_{1,n},\R_{2,n}\} + \left( 1 + \beta \log n\right)$.  Thus for $w=0$ we see (\ref{Daron1}) becomes the desired inequality.  Case (ii): $B_n \le   A_n$. Choosing $w=1$ the rest of the proof is similar. To prove the distinct and similar cases use the worst case bound over $i=1,2,\ldots,n_0$ and observe for all $n \ge n_0$ the action settles on the better of $z_{1,n}$ or $z_{2,n}$. 
}

\end{proof}

Theorem \ref{lem:second} says that if the regrets of experts 1 and 2 differ by at least $\beta\log(n)$ then the regret of the combination will have the same order as the best expert.   For example, if the worst expert has $\Theta(\sqrt{n})$ regret and the better expert has $\Theta(\log{n})$ or $\Theta(1)$ regret then the combination has $\Theta(\log{n})$ or $\Theta(1)$ regret.   When both experts have regret of the same order then the combination will also have regret of that order except perhaps when both have $\Theta(\log(n))$ regret (in which case the worst case regret of $\Theta(\sqrt{n})$ may kick in).

\begin{figure}
\centering
\subfigure[]{
\includegraphics[width=0.45\columnwidth]{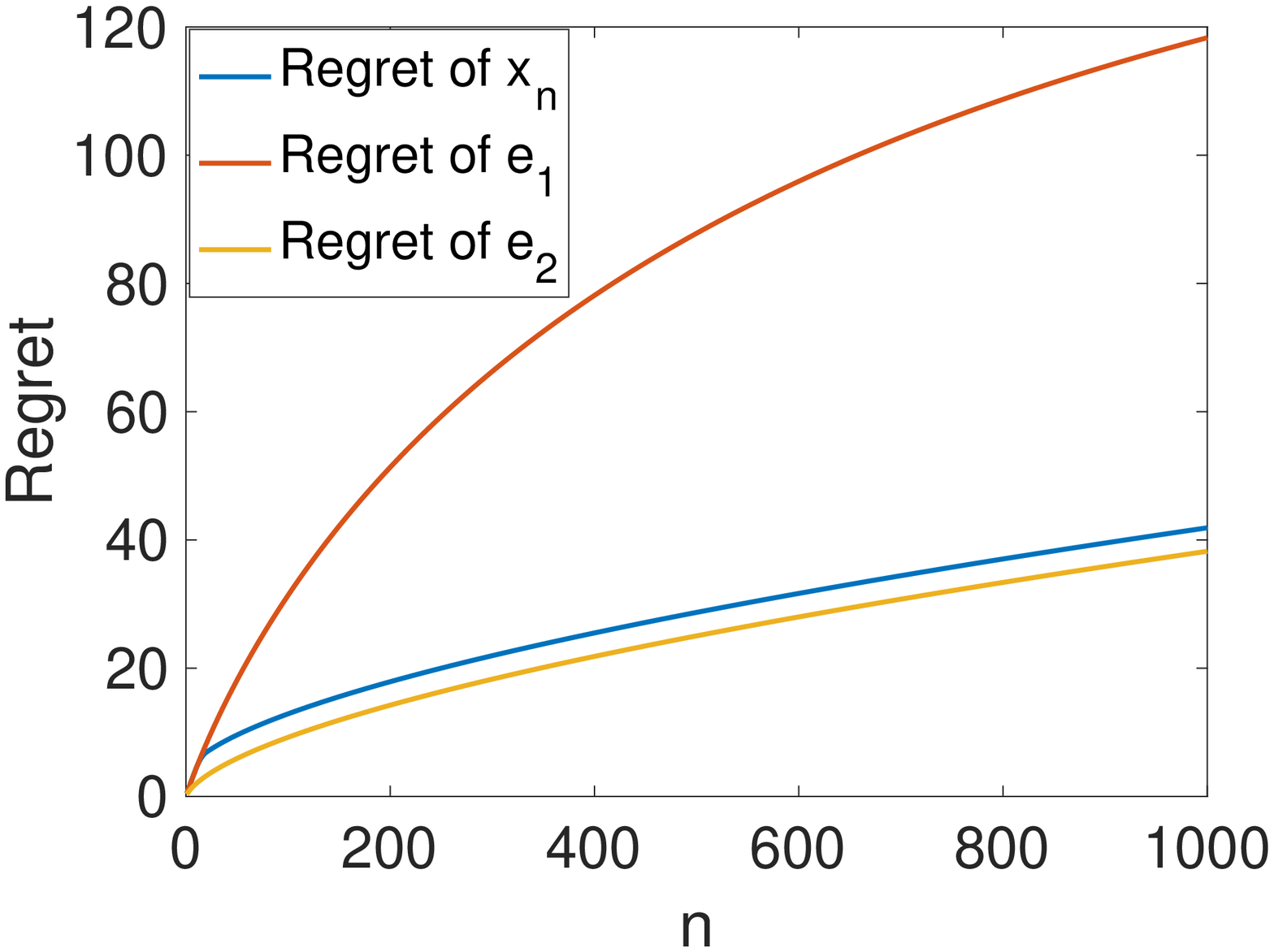}
}
\subfigure[]{
\includegraphics[width=0.45\columnwidth]{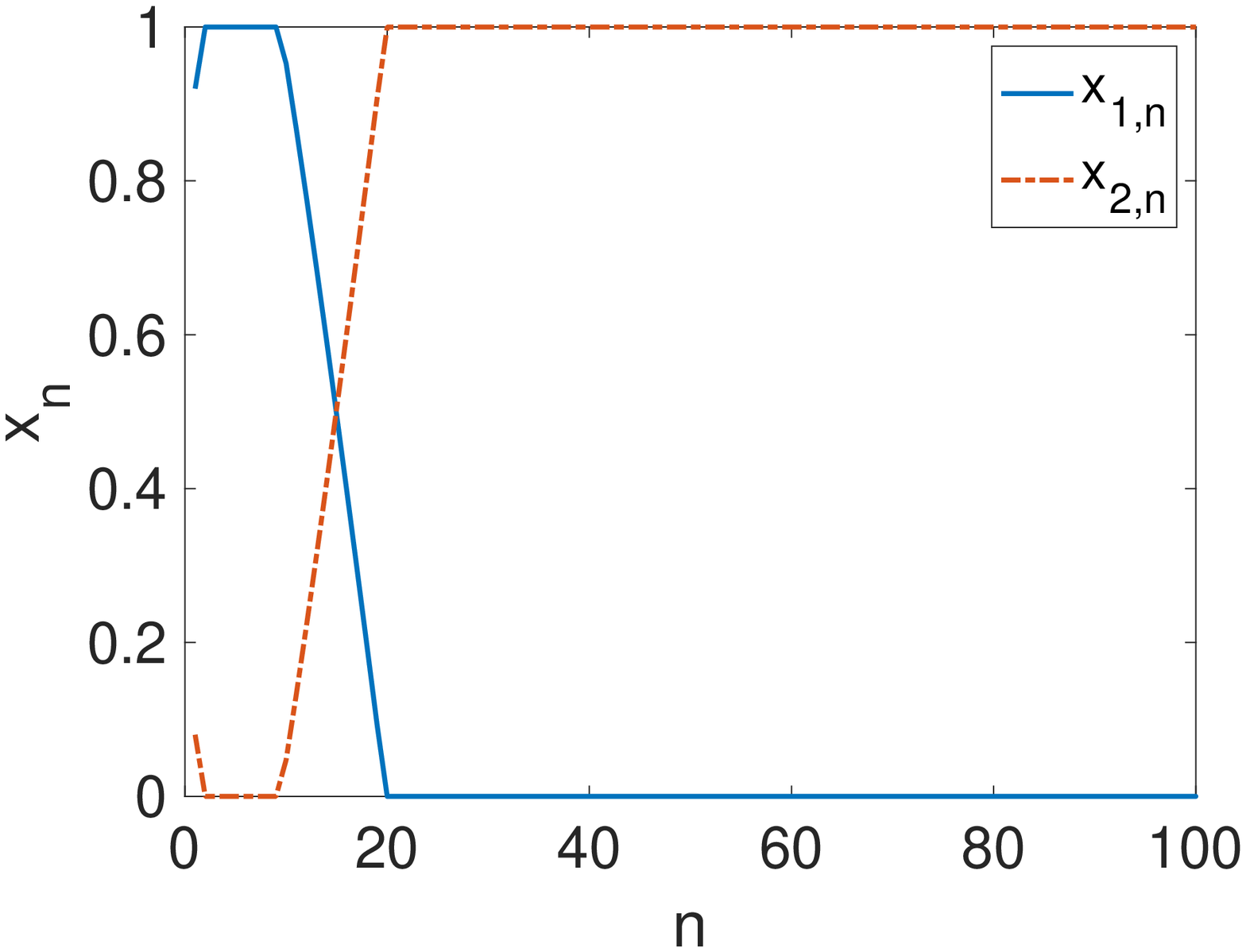}
}
\caption{Example \ref{ex5} combining subgradient algorithms with step sizes proportional to $1/n$ and $1/\sqrt{n}$ using the biased subgradient method.  Left hand plot shows the regret of the combined action $x_n$ and also the regret of expert 1 (step size $0.01/\sqrt{n}$) and expert 2 (step size $0.1/n$).  Right hand plots action $x_n$ taken vs time.}\label{fig:ex5}
\end{figure}
\begin{figure}
\centering
\subfigure[]{
\includegraphics[width=0.45\columnwidth]{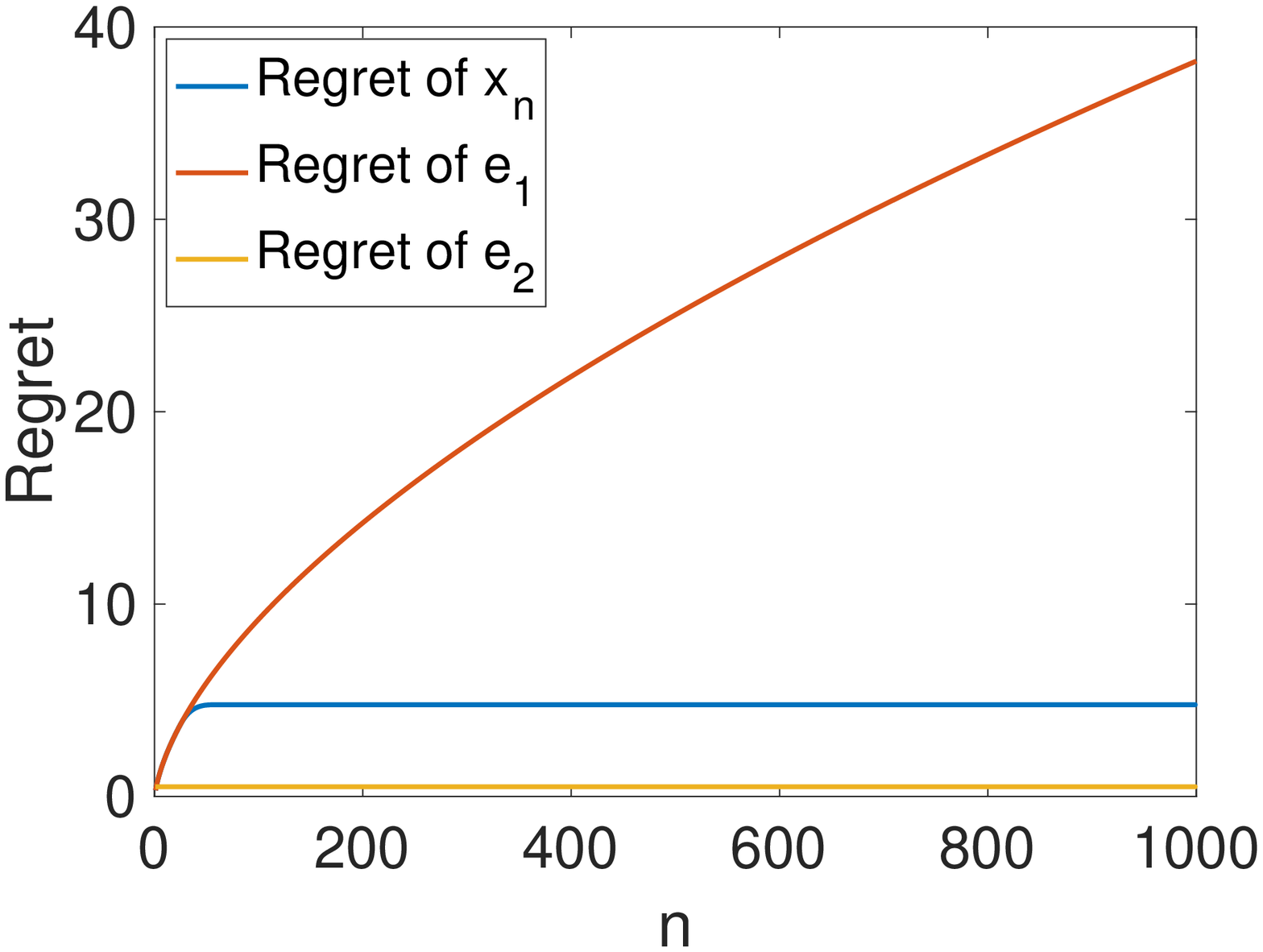}
}
\subfigure[]{
\includegraphics[width=0.45\columnwidth]{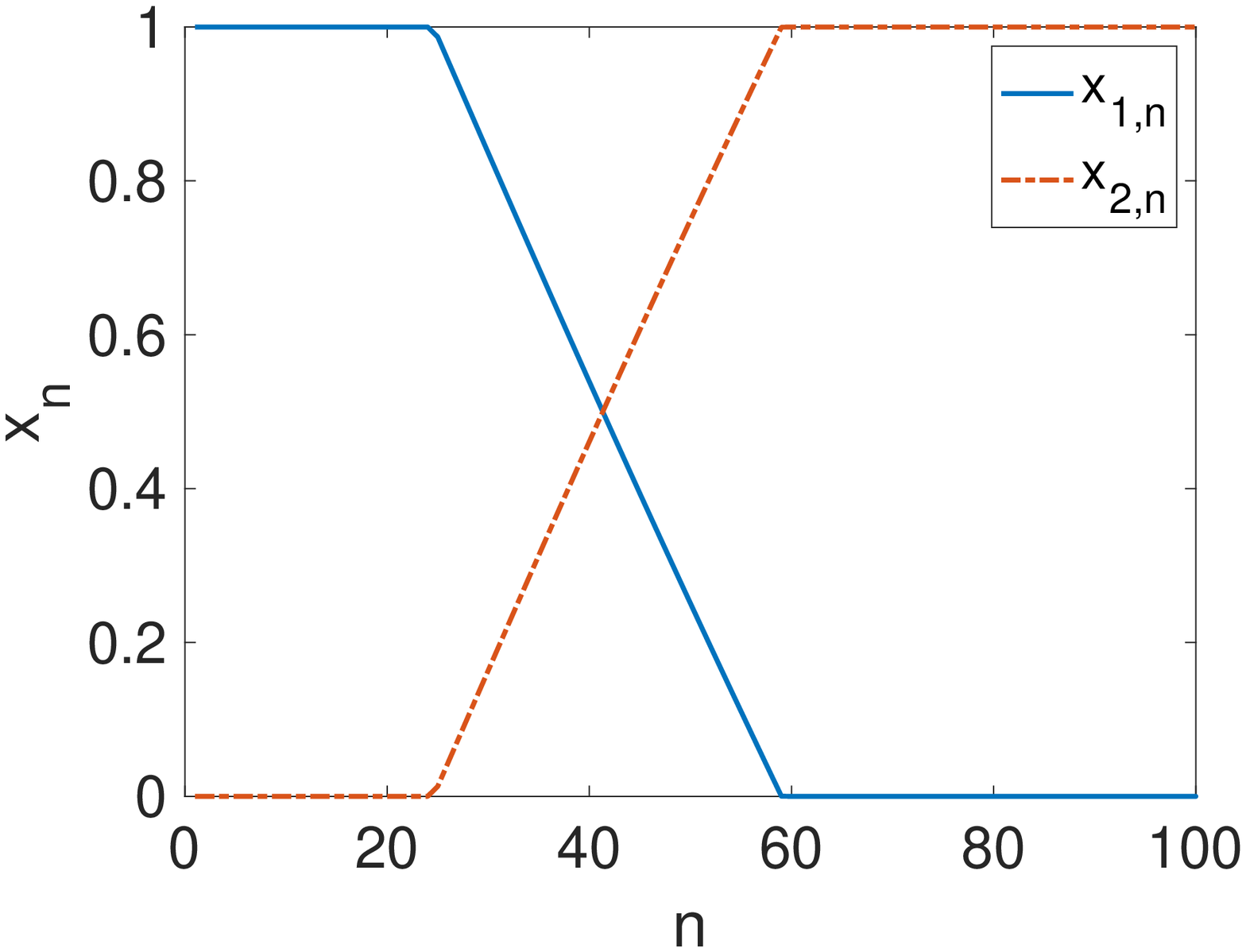}
}
\caption{Example \ref{ex5} combining subgradient algorithms with step sizes proportional to $1/n$.  Left hand plot shows the regret of the combined action $x_n$ and also the regret of expert 1 (step size $0.1/n$) and expert 2 (step size $1/n$).  Right hand plots action $x_n$ taken vs time.}\label{fig:ex5_2}
\end{figure}
\begin{example}\label{ex5}
The subgradient algorithm with step size proportional to $1/n$ achieves $O(\log n)$ regret for strongly convex functions.  However, this step size can lead to $O(n)$ regret in an adversarial setting.  We therefore consider combining the experts generated by the subgradient algorithm with step size proportional to $1/n$ with those generated by subgradient algorithm with step size $1/\sqrt{n}$ (which ensures $O(\sqrt{n})$ worst-case regret).   Figure \ref{fig:ex5} shows example results for cost function $z^2$ (so with loss $\l_i=-2z$).    It can be seen that after about iteration 10 the algorithm switches from expert 1 to expert 2 (i.e. the expert with lower regret) and thereafter settles on this expert.  Figure \ref{fig:ex5_2} shows a second example where both experts use a subgradient algorithm with step size proportional to $1/n$ but one uses step size $0.1/n$ and the other step size $1/n$.    

\end{example}




\subsection{More Than Two Experts}

We can accommodate more than two experts by cascading update (\ref{eq:alg2a}).   For example, to select between three experts we can use the following update,
\begin{align*}
x_{1,i}^{(1)}&=P_\I(\tilde{A}_i+\alpha_{i-1}B_{i-1}^{(1)}),\ 
y_i^{(1)}=z_{1,i}x_{1,i}^{(1)}+z_{2,i}(1-x_{1,i}^{(1)}) \\
x_{1,i}^{(2)}&=P_\I(\tilde{A}_i+\alpha_{i-1}B_{i-1}^{(2)}),\ 
y_i=y_i^{(1)}x_{1,i}^{(2)}+z_{3,i}(1-x_{1,i}^{(2)})
\end{align*}
with $B_i^{(1)}=\sum_{j=1}^{i}\l_j^T(z_{2,j}-z_{1,j})$ and $B_i^{(2)}=\sum_{j=1}^{i}\l_j^T(z_{3,j}-y_i^{(1)})$.   The foregoing analysis carries over directly.


\section{Gap-Like Behaviour in Hedge Algorithm}\label{sec:hedge2}

Consider applying the Hedge algorithm to select between $d=2$ experts with step size $\alpha_i$.  It selects actions as follows,
\begin{align*}
x_{1,i+1}=\frac{e^{-\alpha_i\sum_{j=1}^i\l_i^Tz_{1,j}}}{e^{-\alpha_i\sum_{j=1}^i\l_i^Tz_{1,j}}+e^{-\alpha_i\sum_{j=1}^i\l_i^Tz_{2,j}}},\ x_{2,i+1}=1-x_{1,i+1}
\end{align*}
Dividing through and using the fact that $\sum_{j=1}^i(\l_i^Tz_{2,j}-\l_i^Tz_{1,j})=\R_{2,i}-\R_{1,i}$, this can be rewritten equivalently as,
\begin{align}
x_{1,i+1}=\frac{1}{1+e^{-\alpha_i(\R_{2,i}-\R_{1,i})}},\ x_{2,i+1}=1-x_{1,i+1} \label{eq:hedgeup}
\end{align}
Under update (\ref{eq:hedgeup}) for $x_{1,i+1}$ to reach value 0 or 1 we need $\R_{2,i}-\R_{1,i}\rightarrow \pm\infty$, unlike in Lemma \ref{lem:equil}.  Hence, $x_{1,i+1}$ at best converges only asymptotically to an extreme point of the simplex.   Over any finite time interval it therefore always places weight on both experts and so, on the face of it, it may seem unsuitable for achieving efficient regret.   

{
That said, suppose expert 2 has lower loss than expert 1. The regret for turn $i+1$ relative to expert 2 is  
\begin{align}
\frac{\l_{i+1}^T(z_{1,i+1} - z_{2,i+1})}{1+e^{-\alpha_{i}(\R_{2,i}-\R_{1,i})}} 
\label{re:hedgecond}
\end{align}
Provided $\R_{2,i}-\R_{1,i} \to -\infty$ sufficiently quickly, the above series converges giving 
 $O(1)$ regret relative to expert 2.  In particular, to make each term less than $\gamma_{i}$ it is enough to demand 
\begin{align*}
\R_{1,i} -\R_{2,i}\ge \frac{1}{\alpha_i}\log\left (\frac{\l_{i+1}^T(z_{1,i+1} - z_{2,i+1})}{\gamma_i}-1\right). 
\end{align*}
For example, taking $\alpha_i = 1/\sqrt i$ and the convergent series $\gamma_i = 1/i^2$ the above becomes $\R_{1,i} - \R_{2,i} \ge O(\sqrt i \log(i))$. We summarise these observations in the following lemma. 

\begin{lemma}[Hedge Gap]\label{lem:hedgegap} Suppose all $\l_{i}^T(z_{1,i} - z_{2,i}) \le L$ and for all $i \ge n_0$ we have $\R_{1,i} - \R_{2,i} \ge \frac{1}{\alpha_i}\log\left (\frac{L}{\gamma_i}-1\right)$.  Then the regret $\R_n$ of the Hedge update (\ref{eq:hedgeup}) satisfies $$\displaystyle  \R_n\le \R_{2,n}+\sum_{i=n_0}^n\gamma_i+L\max\{1,n_0\}.$$ 
The same holds with the the roles of $1$ and $2$ reversed.
\end{lemma}

\noindent The above parallels Lemma \ref{lem:two} for the lazy subgradient method, although the details of the gap and the bounds on regret differ significantly.

Now we revisit Example \ref{ex1} in light of Lemma \ref{lem:hedgegap}. For  $\alpha_i = \eta/\sqrt i$   it can be verified that 
$$\alpha_i ( \R_{1,i}- \R_{2,i} )  = \frac{\eta}{\sqrt i} \left(  \sum_{j=1}^i \frac{1}{2 \sqrt j } -\frac{(-1)^i -1}{2}\right)  \to \eta$$ 
Hence any sequence $\gamma_i$ that satisfies the lemma must have $\gamma_i$ bounded from below, and the lemma only gives a $O(n)$ bound.   A more fine-grained analysis can tighten this to an $O(\sqrt{n})$ bound in Example \ref{ex1}.  It can be seen from Figure \ref{fig:ex1_fixed_hedge} that this $O(\sqrt{n})$ upper bound is attained.    

\begin{figure}
\centering
\subfigure[$\alpha_n=2/\sqrt{n}$]{
\includegraphics[width=0.45\columnwidth]{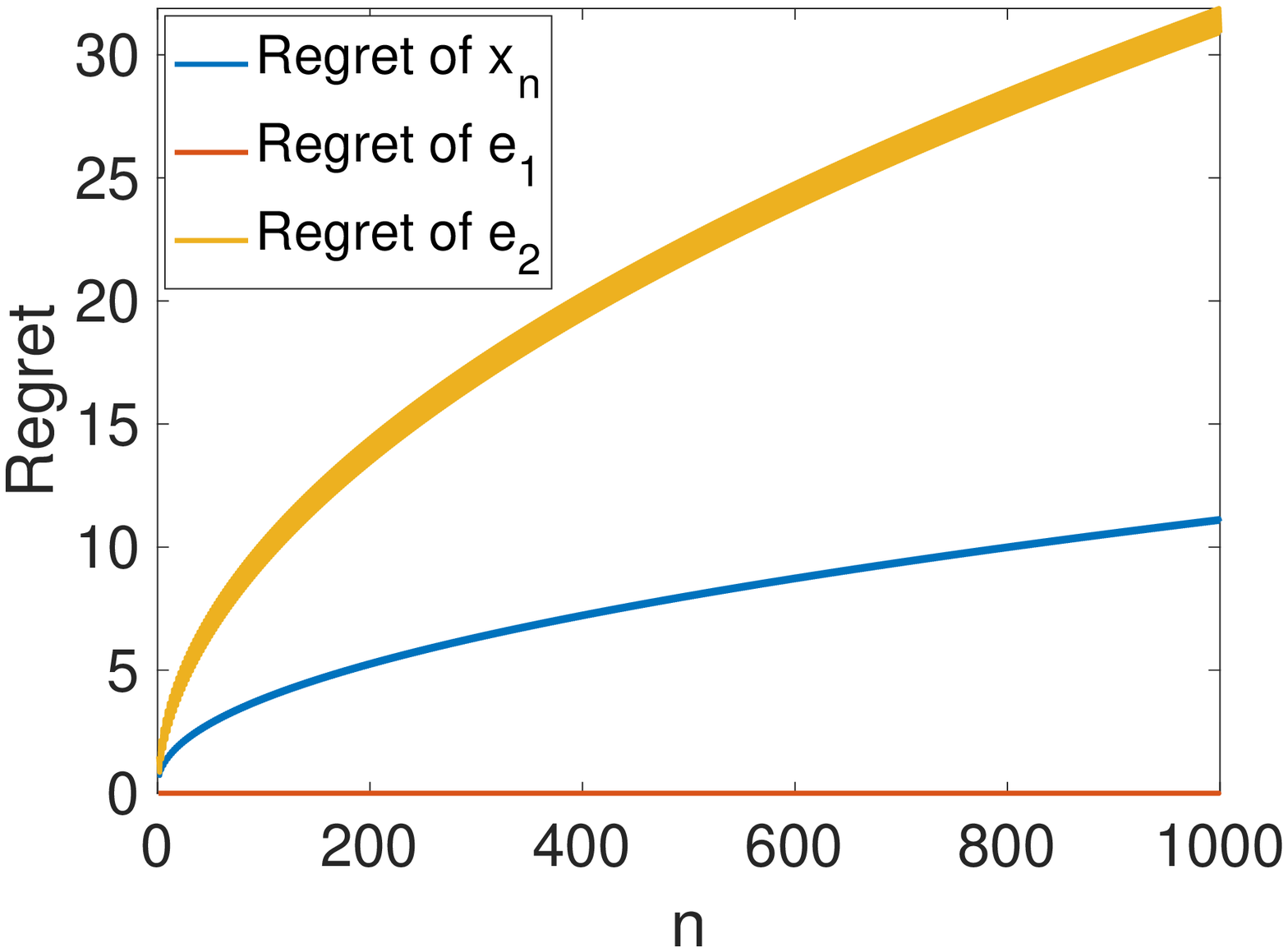}
}
\subfigure[$\alpha_n=5/\sqrt{n}$]{
\includegraphics[width=0.45\columnwidth]{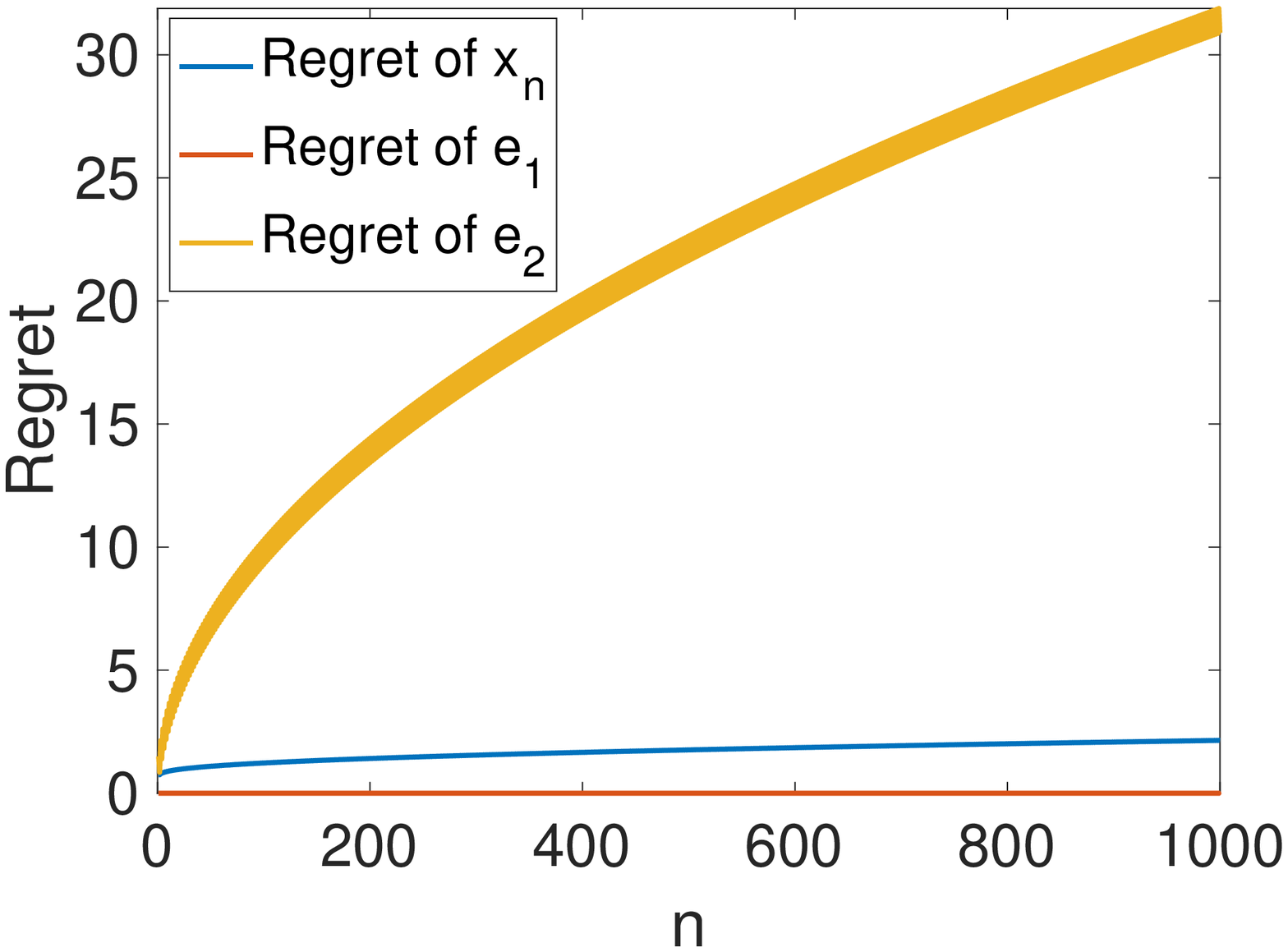}
}
\caption{Performance of the Hedge algorithm in Example \ref{ex1} vs choice of step size. }\label{fig:ex1_fixed_hedge}
\end{figure}

The crux of the difficulty is that to keep the sum $\sum_{i=n_0}^n\gamma_i$ small, $\gamma_i$ has to decrease very quickly, indeed faster than $1/i$ to keep the sum constant, and it is easy to devise examples for which this does not hold.   
Of course we can try to rectify this by adding a bias, similarly to the approach in Section \ref{sec:bias}, or adapting the step size but overall the behaviour of Hedge seems messier than that of the lazy subgradient algorithm due to the inabilility of Hedge to settle on an extreme point in finite time.
}

\section{Manipulating Initial Conditions in Prod and Hedge Algorithms}\label{sec:prodhedge}

The closest work to this is probably that of \cite{ABprod}.   We consider their approach in detail next.  In summary, the mechanism used is substantially different from the gap mechanism in the lazy subgradient approach.  Instead a special choice of initial conditions is used to bias the Prod algorithm towards a favoured expert (a similar approach can also be used with the Hedge algorithm, as we show in Section \ref{sec:hedge}).   The result is highly asymmetrical performance and so it is important to know in advance which expert potentially has lower regret and to know the time horizon in advance.

\subsection{Prod Algorithm}
The Prod algorithm introduced by \cite{cesa-bianchi_improved_2007} uses the following update when there are two actions on the simplex,
\begin{align*}
x_i=w_i/(w_{1,i}+w_{2,i}),\ w_{1,i} = w_{1,i-1}(1+\eta u_{1,i}),\ w_{2,i} = w_{2,i-1}(1+\eta u_{2,i})
\end{align*}
where $\eta>0$ is a design parameter and $u_{k,i}$ is the reward (rather than loss) gained by taking action $k$ at step $i$.   While \cite{cesa-bianchi_improved_2007} consider initial values $w_{1,1}=w_{2,1}=1$ their analysis is readily generalised to other initialisations.  In particular, selecting $w_{1,1}>0$, $w_{2,1}=1-w_{1,1}$ then the analysis of \cite{cesa-bianchi_improved_2007} shows that the cumulative reward satisfies,
\begin{align}
 \sum_{i=1}^n u_i^Tx_i \ge \max\{U_1,U_2\} \label{eq:prodreward}
\end{align}
where $U_k:=\frac{\log w_{k,1}}{\eta} + \sum_{i=1}^n u_{k,i}-\eta \sum_{i=1}^n u_{k,i}^2$.  Following \cite{ABprod}, select $u_{1,i}=\l_i^T(z_{2,i}-z_{1,i})$ and $u_{2,i}=0$ (thus $w_{2,i}=w_{2,1}$ for all $i$).  Plugging these choices into (\ref{eq:prodreward}) and rearranging then yields {
\begin{align*}
\R_n \le \min\left\{\R_{1,n}-\frac{\log w_{1,1}}{\eta} +\eta C,\R_{2,n}-\frac{\log (1-w_{1,1})}{\eta} \right\}
\end{align*}
where $C$ upper bounds $\sum_{i=1}^n \l_i^T(z_{2,i}-z_{1,i})^2$ (e.g. select $C=n\max_i |\l_i^T(z_{2,i}-z_{1,i})^2|$), $\R_n=\sum_{i=1}^n \l_i^T(z_{1,i}x_{1,i}+z_{2,i}(1-x_{1,i})-y^*)$ and $\R_{k,n}=\sum_{i=1}^n \l_i^T(z_{k,i}-y^*)$.  Selecting $\eta=\gamma/\sqrt{C}$ with $\gamma>0$ it follows that 
\begin{align*}
\R_n \le \min\left \{\R_{1,n}-\frac{\log w_{1,1}}{\gamma}\sqrt{C} +\gamma \sqrt{C},\R_{2,n}-\frac{\log (1-w_{1,1})}{\eta} \right \}
\end{align*}
Observe that the regret $\R_n$ appears to scale with $\sqrt{C}$ and that in general we expect $C$ to scale with $n$.   However, \cite{ABprod} make the key observation that $-\frac{\log (1-\eta)}{\eta} \le 2\log 2$ for $\eta\in(0,1/2)$.   Hence, selecting $w_{1,1}=\eta = \gamma /\sqrt C$ yields 
\begin{align}
\R_n \le \min \left \{\R_{1,n}-\frac{\log\gamma-\log\sqrt{C}}{\gamma}\sqrt{C} +\gamma \sqrt{C},\R_{2,n}+2\log 2 \right \}\label{eq:prod}
\end{align}
That is, this special choice of initial condition removes the scaling with $\sqrt{C}$ in the second term on the RHS, which becomes $\R_{2,n}+2\log 2$.  Selecting $\gamma=\sqrt{\log C}/2$, which corresponds to the $(A,B)$-Prod algorithm of \cite{ABprod}, simplifies the bound to \begin{align*}\R_n &\le \min\left \{\R_{1,n}+\left (2 \log 2 + \frac{1}{2} \right )\sqrt{C\log C},\R_{2,n}+2\log 2 \right \}\\ &\le \min\{\R_{1,n}+2\sqrt{C\log C},\R_{2,n}+2\log 2 \} 
\end{align*} }

The key insight here is that it is the choice of initial condition that is doing all the heavy lifting.  This is not just an artefact of the analysis but reflects actual algorithm behaviour, as illustrated by the following example.

\begin{figure}
\centering
\subfigure[]{
\includegraphics[width=0.45\columnwidth]{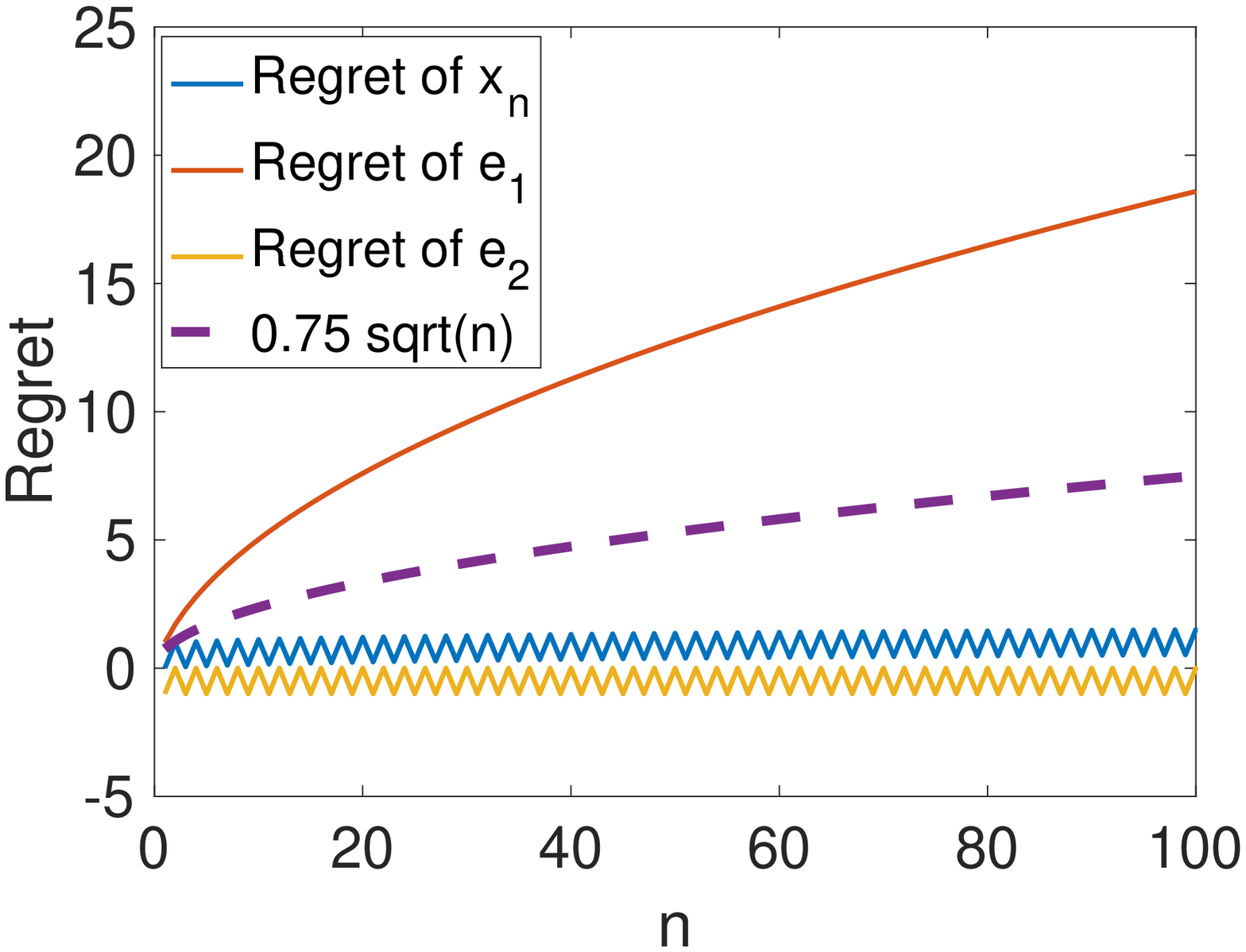}
}
\subfigure[]{
\includegraphics[width=0.45\columnwidth]{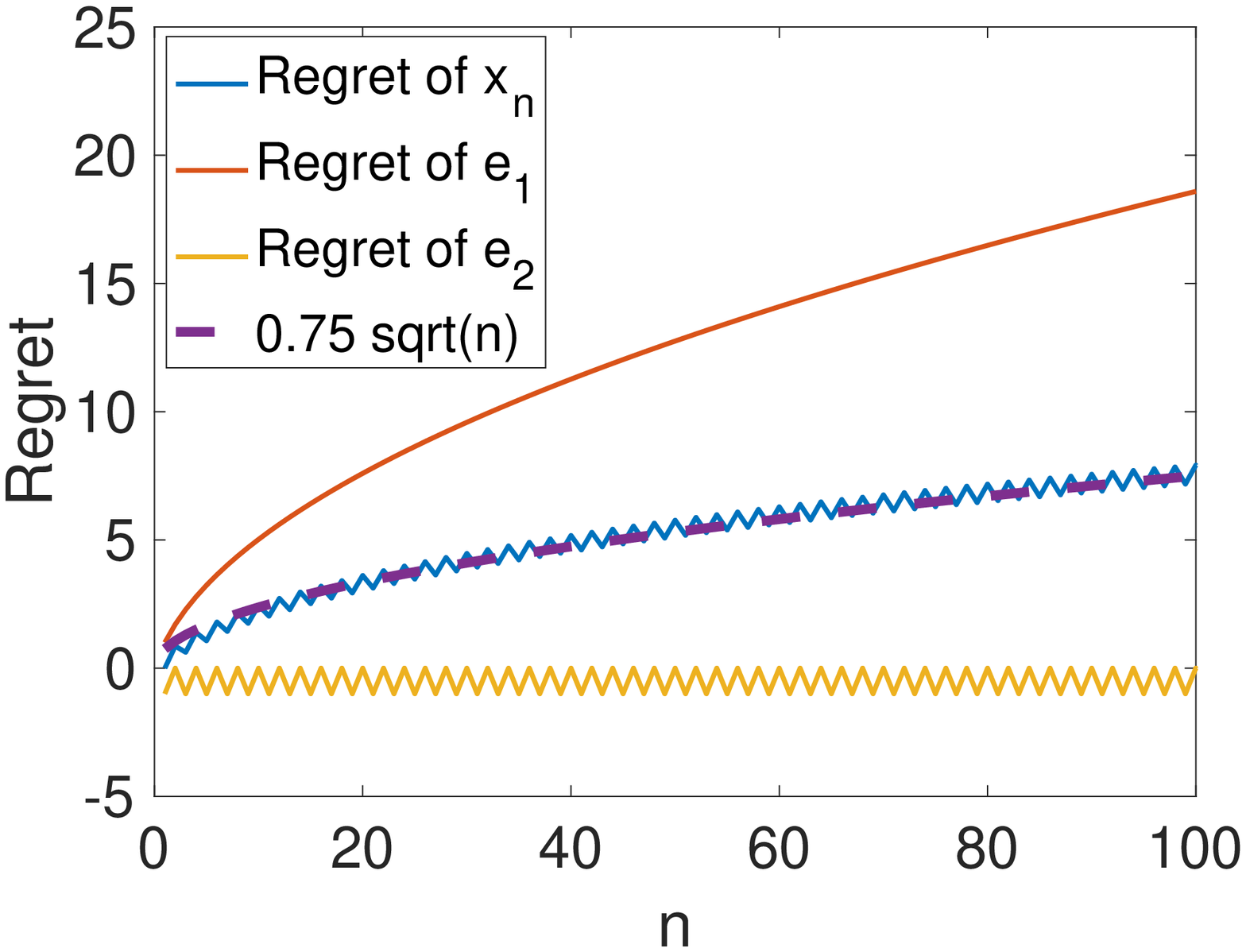}
}
\caption{Illustrating the sensitivity of $(A,B)$-Prod to choice of initial conditions.   In left-hand plot the initial condition is $w_{1}=(\eta,1-\eta)$, which gives constant loss.  In the right-hand plot the initial condition is changed to be $w_1=(0.5,0.5)$ while keeping everything else unchanged.  It can be seen that this change results in $\Theta(\sqrt{n})$ loss.}\label{fig:prod_initconds}
\end{figure}
\begin{example}\label{prod_initconds}
Suppose $\l_{1,i}=1/\sqrt{i}$, $\l_{2,i}=(-1)^{i+1}$ and $z_{1,i}=(1,0)$, $z_{2,i}=(0,1)$ and $y^*=(0,1)$.   Figure \ref{fig:prod_initconds}(a) shows the regret when using the $(A,B)$-Prod method with initial condition $w_{1,1}=\eta$, $w_{2,1}=1-\eta$.   Figure \ref{fig:prod_initconds}(b) shows the regret when the initial condition is changed to $w_{1,1}=w_{2,1}=0.5$.  It can be seen that in the first case the regret is $\Theta(1)$ while in the second the regret is $\Theta(\sqrt{n})$.  Note that the only change made here is in the initial condition.   

\end{example}

Observe also that the $(A,B)$-Prod regret bound (\ref{eq:prod}) is asymmetric.  Namely, it is useful when we have a situation where one expert has $\Theta(\sqrt{n})$ regret and the other expert may have lower regret if the data is favourable, plus we know in advance which of the experts may have lower regret.    We can then order the experts so that the expert which may have low regret corresponds to $z_{1,i}$ and the $\Theta(\sqrt{n})$ expert corresponds to $z_{2,i}$.   This ordering of the experts matters, namely if $z_{2,i}$ happens to achieve less than $\Theta(\sqrt{n})$ regret and $z_{2,i}$ has $\Theta(\sqrt{n})$ regret then the regret of $(A,B)$-Prod may be $\Theta(\sqrt{n})$.  The following example shows that this sensitivity to ordering is not just a deficiency of the (analysis but can actually occur.  

\begin{figure}
\centering
\subfigure[]{
\includegraphics[width=0.45\columnwidth]{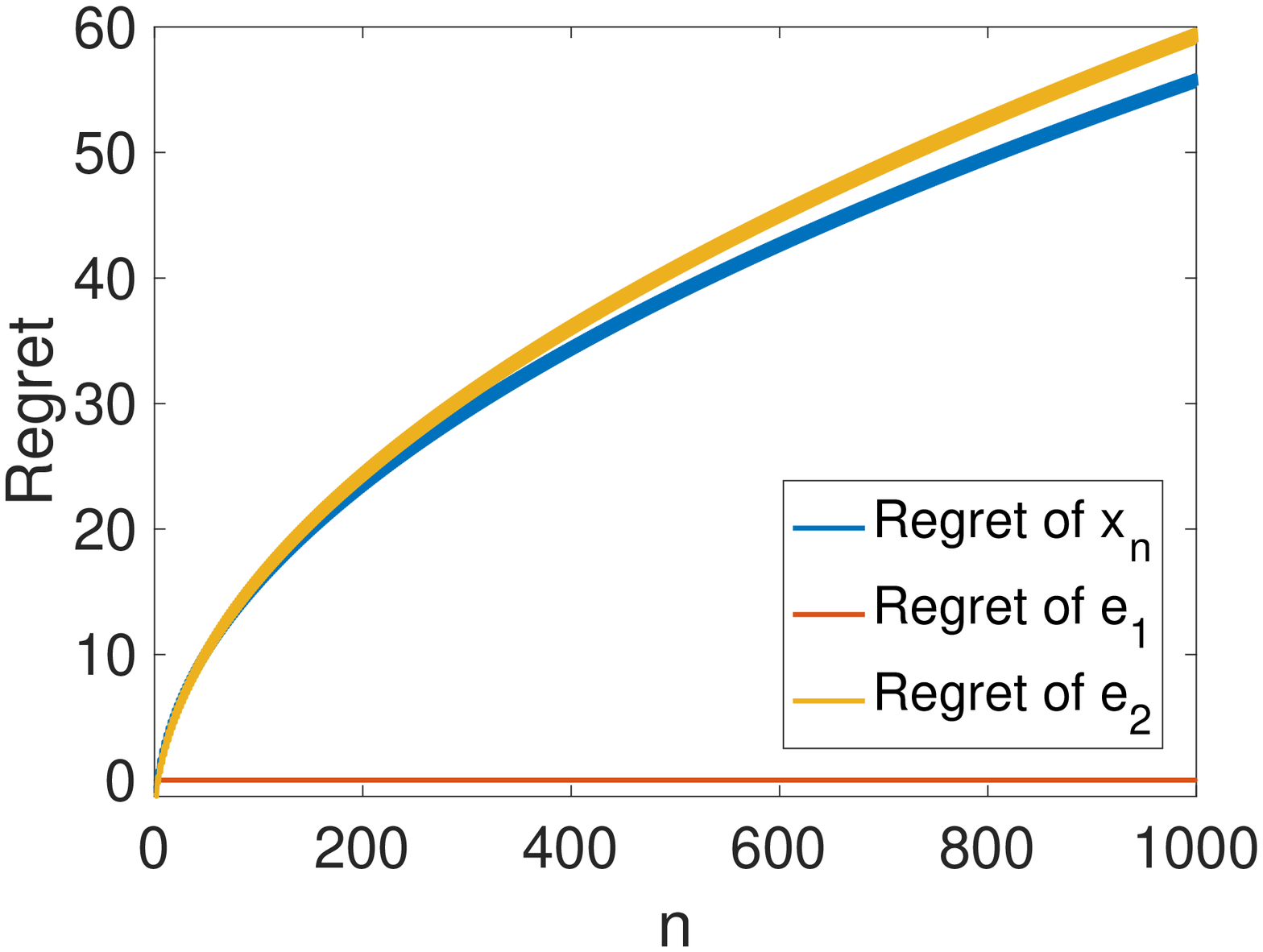}
}
\subfigure[]{
\includegraphics[width=0.45\columnwidth]{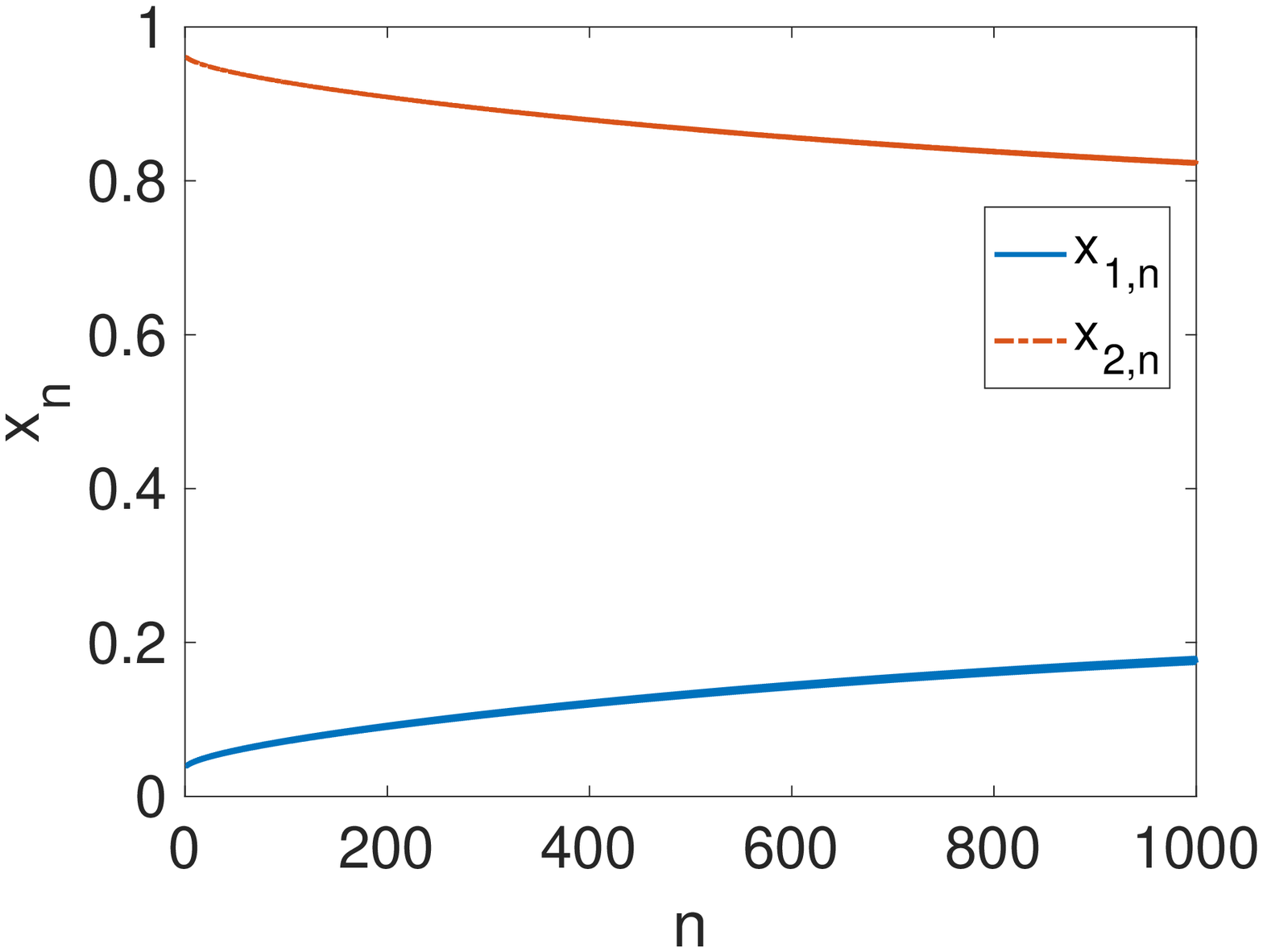}
}
\caption{Example \ref{ex3} of $(A,B)$-Prod asymmetry.  Left hand plot shows the regret of the combined action $x_n$ taken by $(A,B)$-Prod and also the regret of experts 1 and 2 with respect to expert 1.  Right hand plots action $x_n$ taken vs time.}\label{fig:ex3}
\end{figure}
\begin{figure}
\centering
\subfigure[]{
\includegraphics[width=0.45\columnwidth]{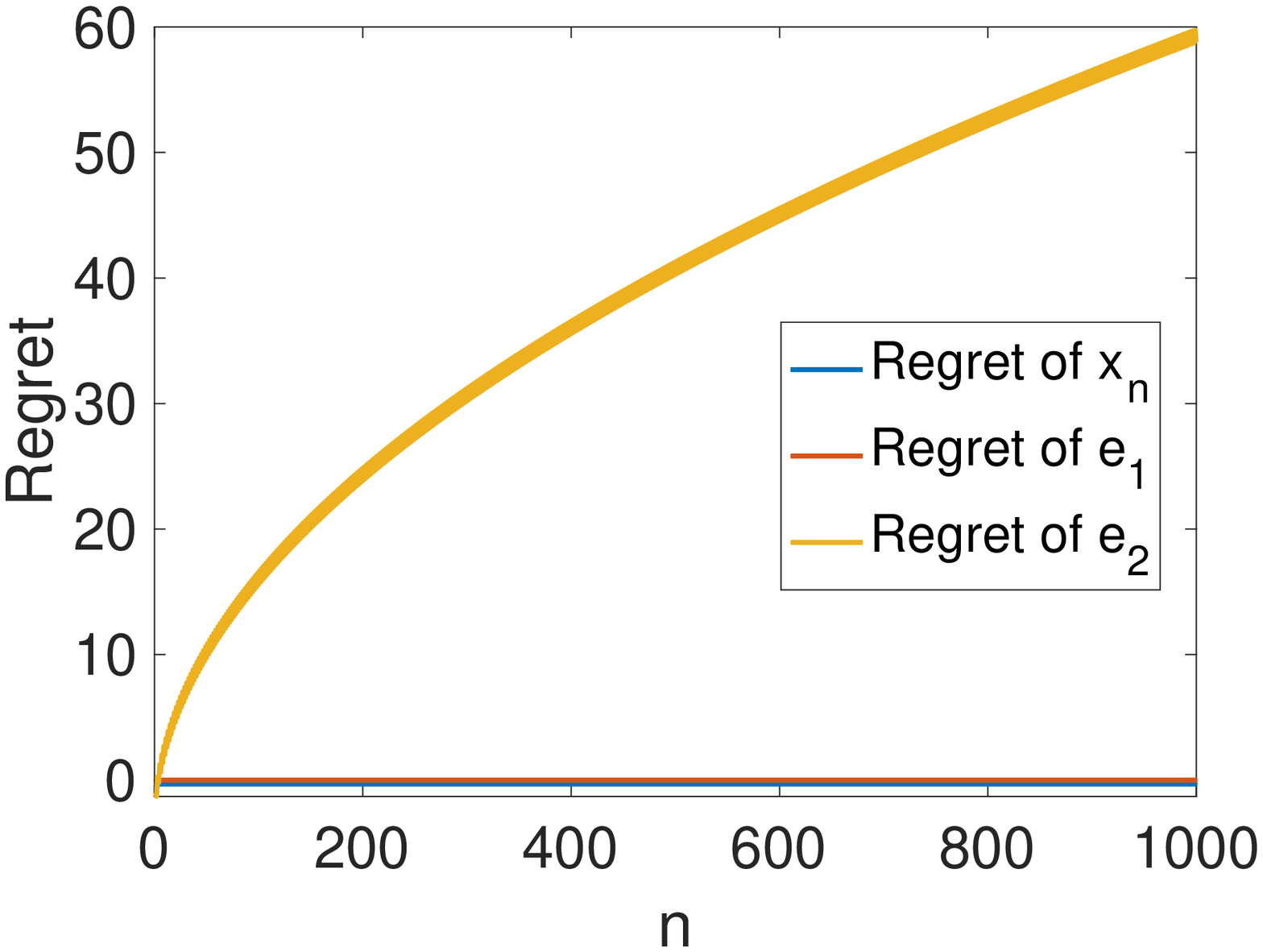}
}
\subfigure[]{
\includegraphics[width=0.45\columnwidth]{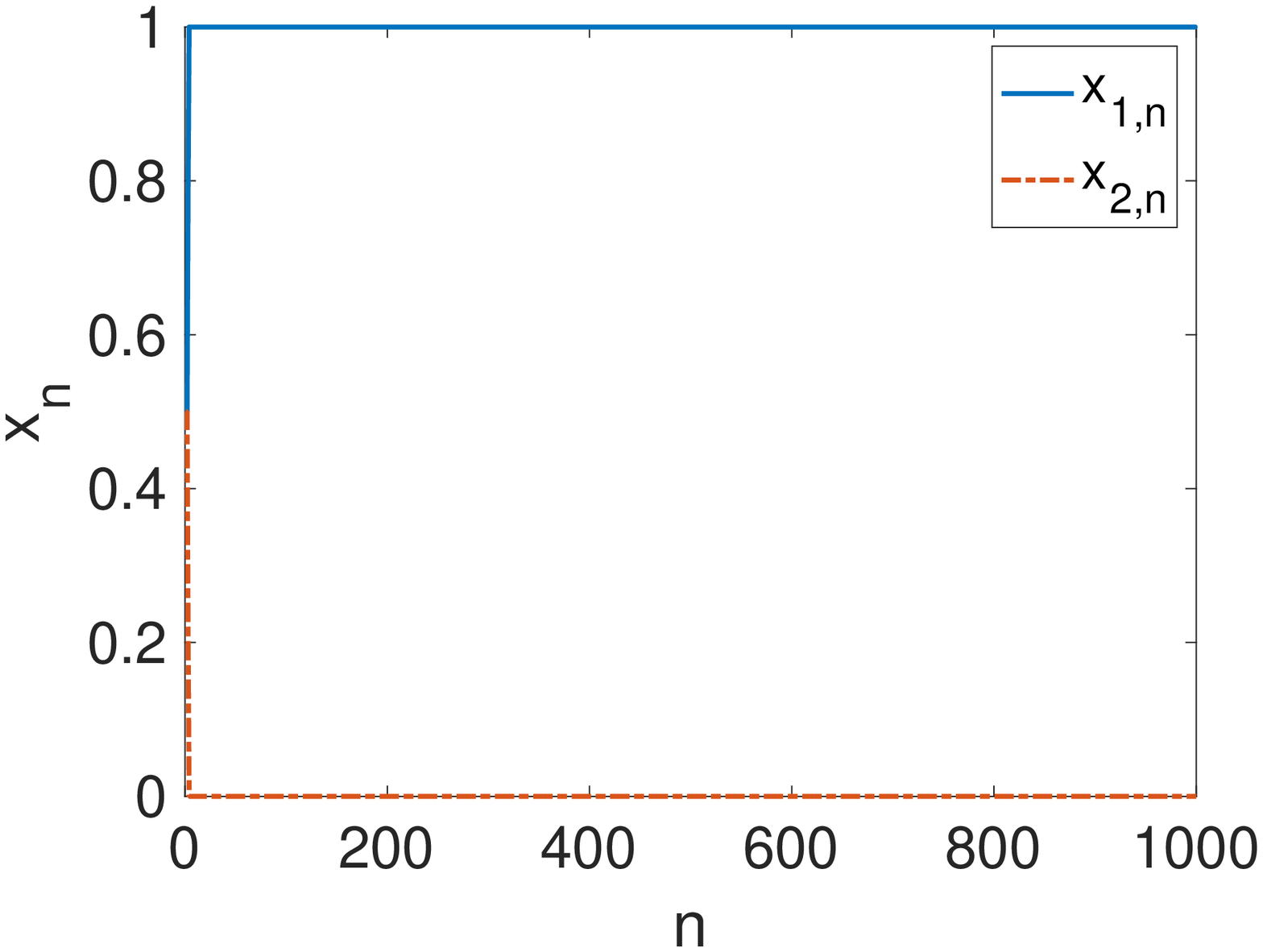}
}
\caption{Performance of biased subgradient method (\ref{eq:alg2a}) with $A_i=-(\sqrt{i}+\log i)$ in Example \ref{ex3}.  Left hand plot shows the regret of the combined action $x_n$ taken by the biased subgradient method and also the regret of experts 1 and 2 with respect to expert 1.  Right hand plots action $x_n$ taken vs time.}\label{fig:ex3_subgrad}
\end{figure}
\begin{example}\label{ex3}
Consider Example \ref{prod_initconds} but with the losses flipped i.e. $\l_{1,i}=(-1)^{i+1}$ and $\l_{2,i}=1/\sqrt{i}$.   Now $y^*=(1,0)$ and the regret of the second expert is $\Theta(\sqrt{n})$.  Figure \ref{fig:ex3}(a) shows the regret when these experts are combined using the $(A,B)$-Prod method.  It can be seen that the regret grows as $\Theta(\sqrt{n})$ even though expert 1 has no regret.    Figure \ref{fig:ex3}(b) plots $x_n$ vs time.  It can be seen that the difficulty arises because $(A,B)$-Prod moves the action away from the $(0,1)$ point too slowly.  Note that Theorem \ref{lem:first} says that the behaviour is almost symmetric when using the biased subgradient method (\ref{eq:alg2a}) to combine experts.  In particular, if $z_{2,i}$ happens to achieve less than $\Theta(\sqrt{n})$ regret and $z_{2,i}$ has $\Theta(\sqrt{n})$ then the biased subgradient method will achieve the lower regret.  Figure \ref{fig:ex3_subgrad} illustrates this behaviour.
\end{example}



\subsection{Hedge Algorithm}\label{sec:hedge}
The above discussion for the Prod algorithm carries over to the Hedge algorithm (\cite{hedge97}) as follows.  Hedge with rewards uses the following update when there are two actions on the simplex { 
\begin{align*}
x_{1,i}=\frac{w_{1,i}}{w_{1,i}+w_{2,i}}\qquad  x_{2,i}=1-x_{1,i} \qquad  w_{1,i} = w_{1,1}e^{\eta \sum_{j=1}^{i-1}u_{1,j}} \qquad w_{2,i} = w_{2,1}e^{\eta \sum_{j=1}^{i-1}u_{2,j}}
\end{align*}
with $w_{1,1}>0$ and $w_{2,1}>0$.  For the modified Hedge update $x_{1,i}$ and $x_{1,i}$ are the same as above, but we change the reward $u_{1,j}$ used in the exponent to $u_{1,j}-\eta u_{1,j}^2$ giving weights
\begin{align}  w_{1,i} = w_{1,1}e^{\eta \sum_{j=1}^{i-1}(u_{1,j}-\eta u_{1,j}^2)} \qquad w_{2,i} = w_{2,1}e^{\eta \sum_{j=1}^{i-1}(u_{2,j}-\eta u_{2,j}^2)}.
\label{eq:exp3mod}
\end{align}
The Second-Order Hedge algorithm exhibits the following behaviour:}
\begin{lemma}[Second-Order Hedge]\label{lem:hedge}
For initial condition $w_{1,1}>0$, $w_{2,1}=1-w_{1,1}$ the cumulative reward of the modified Hedge update (\ref{eq:exp3mod}) satisfies,
\begin{align*}
\sum_{i=1}^{n-1} (u_{1,i}x_{1,i} + u_{2,i}x_{2,i}) \ge \max\{U_1,U_2\}
\end{align*}
where $U_k:= \frac{\log w_{k,1}}{\eta} +\sum_{i=1}^n u_{k,i}-\eta \sum_{i=1}^n u_{k,i}^2$.
\end{lemma}
\begin{proof}
Let $W_i=w_{1,i}+w_{2,i}$.  Then,
\begin{align}
\log \frac{W_{n+1}}{W_1} &\ge \frac{\log w_{1,n+1}}{W_1}\ge-\log W_1 + \log w_{1,1} +\sum_{i=1}^n \eta (u_{1,i}-\eta u_{1,i}^2) \label{eq:w1exp3}
\end{align}
Note that by assumption $W_1=w_{1,1}+w_{2,1}=1$ and so $\log W_1=0$.  We also have that,
\begin{align*}
\frac{W_{i+1}}{W_i} 
&=x_{1,i}e^{\eta u_{1,j}-\eta^2 u_{1,j}^2}+x_{2,i}e^{\eta u_{2,j}-\eta^2 u_{2,j}^2}\\
&\stackrel{(a)}{\le} x_{1,i}(1+\eta u_{1,j} ) + x_{2,i}(1+\eta u_{2,j} )
\stackrel{(b)}{=}1+\eta (x_{1,i} u_{1,j} +x_{2,i} u_{2,j}) 
\end{align*}
where inequality $(a)$ follows from the identity $e^{x-x^2} \le 1+x$ (e.g. see  \cite{cesa-bianchi_improved_2007}) and equality $(b)$ from the fact that $x_{1,i}+x_{2,i}=1$.  Hence,
\begin{align*}
\log \frac{W_{n+1}}{W_1}& = \log\Pi_{i=1}^n \frac{W_{i+1}}{W_i} \le \eta \sum_{i=1}^{n-1}x_{1,i} u_{1,j} +x_{2,i} u_{2,j}
\end{align*}
where we have used the fact that $\log(1+x)\le x$.   Combining this expression with (\ref{eq:w1exp3}) yields the stated result.
\end{proof}
Lemma \ref{lem:hedge} is identical to (\ref{eq:prodreward}), and so the previous the analysis for the Prod algorithm now carries over unchanged and we can get the same asymmetric bound by a special choice of initial conditions.   Note that the existence of a close link between Hedge and Prod has also been previously noted by \cite{koolen_second-order_2015}, although the connection with $(A,B)$-Prod seems to be new.

\section{Summary and Conclusions} {
Standard online learning algorithms often fail to achieve efficient regret in easy examples.  However the biased lazy subgradient algorithm (\ref{eq:alg2a_subgrad}) can achieve efficient regret in such examples.} The Prod/Second-Order Hedge algorithms with appropriate choice of initial condition can also achieve efficient regret, but in a less clean way than with the lazy subgradient algorithm.

In this work we consider $\Theta(1/\sqrt{n})$ step sizes since these ensure $O(\sqrt{n})$ worst-case regret.  However, in light of work such as that of \cite{gaillard_second-order_2014} an obvious open question is whether use of an adaptive step size would yield improved efficiency, in particular shrinking of the regret gap required for a case to be counted as ``easy''. 

\section*{Acknowledgements}
 This work was supported by Science Foundation Ireland grant 16/IA/4610.

\section*{Appendix}

The following are straightforward variations on standard results but we were unable to find a suitable existing result in the literature that covers the exact conditions we need.

{
\begin{lemma}[Strong convexity]\label{lem:strong}
The actions generated by the biased subgradient update (\ref{eq:alg2a}) satisfy $\|x_{i+1}-x_i\|  \le   \frac{1+|a_{i+1}| }{\sqrt i} + \frac{1}{4   i} $.
\end{lemma}
\begin{proof}
We adapt the proof of of Lemma 2.10 in \cite{Purple1} to the present setting where the regulariser changes at each iteration.  Observe that $x_{i+1}=\arg\min_{x\in \S} \|x+\alpha_{i}(A_{i},B_{i})\|^2$.  Expanding the square and dropping terms that do not depend on $x$ it follows that $x_{i+1}=\arg\min_{x\in \S} \|x\|^2+2\alpha_{i}(A_{i},B_{i})^Tx=\arg\min_{x\in \S} F_i(x)$ for $F_i(x) = R_{i}(x)+(A_{i},B_{i})^Tx$ and $R_i(x)=\frac{\sqrt{i}}{2}\|x\|^2$.   
Since each $\|u\|^2=\|u-x_i+x_i\|^2 = \|u-x_i\|^2 +2x_i^T(u-x_i)+\|x_i\|^2$ the definition of $R_i$ gives
\begin{align}
R_i(u) -R_i(x_i)  = \sqrt{i}x_i^T(u-x_i) + \frac{\sqrt{i}}{2}\|u-x_i\|^2\\
F_i(u)-F_i(x_i) = (\sqrt{i}x_i+(A_{i},B_{i}))^T(u-x_i)+ \frac{\sqrt{i}}{2}\|u-x_i\|^2\\
=  \partial F_i(x_i)^T(u-x_i)+ \frac{\sqrt{i}}{2}\|u-x_i\|^2\label{SC}
\end{align}
where the last line follows from $\partial F_i(x_i)^T(u-x_i) = (\sqrt{i}x_i+(A_{i},B_{i}))^T(u-x_i)$. We claim the first term on the right is nonnegative.  

Since $x_i$ is a minimiser of $F_i$ the negative of the gradient $-\partial F_i(x_i)$ is normal to the domain $\S$. Since $\S$ is convex it is contained in the half-space $\{x \in \mathbb R ^2:    -\partial F_i(x_i) ^T x \le \partial  F_i(x_i) ^T x_i  \}$. In particular $-\partial F_i(x_i) ^T u \le \partial  F_i(x_i) ^T x_i$ and so $\partial F_i(x_i) ^T (u-x_i) \ge 0$ as required. Thus (\ref{SC}) gives 
\begin{align}
F_i(u)-F_i(x_i) \ge \frac{\sqrt{i}}{2}\|u-x_i\|^2\label{SC1}
\end{align}
Setting $u=x_{i+1}$ we get $F_{i}(x_{i+1}) -F_{i}(x_{i})\ge\frac{\sqrt{i}}{2}\|x_{i+1}-x_i\|^2$. Since (\ref{SC1}) holds for all $i$ and $u$ it holds for $i+1$ and  $u = x_i$. Hence we get  $F_{i+1}(x_i)- F_{i+1}(x_{i+1})\ge \frac{\sqrt{i+1}}{2}\|x_i-x_{i+1}\|^2$.  Summing these two inequalities and rearranging gives 
\begin{align*}F_{i}(x_{i+1})  - F_{i+1}(x_{i+1})+ F_{i+1}(x_i)-F_{i}(x_{i})   \ge\frac{\sqrt{i}+\sqrt{i+1}}{2}\|x_{i+1}-x_i\|^2 \ge \sqrt{i}\|x_{i+1}-x_i\|^2.
\end{align*}
For the first pair on the left $F_{i}(x_{i+1})-F_{i+1}(x_{i+1})=-(a_{i+1},b_{i+1})x_{i+1}+\frac{1}{2}(\sqrt{i}-\sqrt{i+1})\|x_{i+1}\|^2$. For the second pair $F_{i+1}(x_i)-F_{i}(x_{i})=(a_{i+1},b_{i+1})x_i+\frac{1}{2}(\sqrt{i+1}-\sqrt{i})\|x_{i}\|^2$.  Hence
\begin{align}
\sqrt{i}\|x_{i+1}-x_i\|^2 
&\le 
(a_{i+1},b_{i+1})(x_i-x_{i+1}) +\frac{1}{2}(\sqrt{i}-\sqrt{i+1})(\|x_{i+1}\|^2-\|x_{i}\|^2)\\
&\le \|(a_{i+1},b_{i+1})\|\|x_i-x_{i+1}\| +\frac{1}{\sqrt{i+1}+\sqrt i}\frac{ \|x_{i}\|^2-\|x_{i+1}\|^2}{2}\\
&\le  ( |a_{i+1}| +1)\|\|x_i-x_{i+1}\| +\frac{1}{2 \sqrt i}\frac{ \|x_{i}\|^2-\|x_{i+1}\|^2}{2}
\end{align}
For the last term we use the parallelogram law 
\begin{align*}\|x_{i}\|^2 - \|x_{i+1}\|^2 = \frac{(x_{i}+x_{i+1})^T(x_{i}-x_{i+1})}{2}\le   \frac{\|x_{i}+x_{i+1}\|\|x_{i}-x_{i+1}\|}{2}\\  \le  \frac{(\|x_{i}\|+ \|x_{i+1}\|)\|x_{i}-x_{i+1}\|}{2} \le \|x_{i}-x_{i+1}\|
\end{align*} to get  $ \displaystyle 
\sqrt{i}\|x_{i+1}-x_i\|^2 \le \left ( |a_{i+1}| +1 + \frac{1}{4 \sqrt i} \right )\|x_i-x_{i+1}\| $
When $\|x_i-x_{i+1}\|=0$ the result is trivial.  Otherwise divide through by $\sqrt i \|x_i-x_{i+1}\|$ to get $ \displaystyle 
 \|x_{i+1}-x_i\|  \le   \frac{|a_{i+1}| +1}{\sqrt i} + \frac{1}{4   i} $ as required.
\end{proof}}


\begin{lemma}[FTRL]\label{lem:ftrl} 
Under update (\ref{eq:alg2a}) with $\alpha_i=1/\sqrt{i}$ the regret {satisfies} $$\sum_{i=1}^n (a_i,b_i)^T(x_i-x^*) \le \R_n(x^*)-\R_1(x_2)+\sum_{i=1}^{n}(a_i,b_i)^T(x_i-x_{i+1})$$ where $R_i(x)=\frac{\sqrt{i}}{2}\|x\|^2$.
\end{lemma}
\begin{proof}
We follow the usual approach, slightly generalised to encompass our setting.  Note that $\sum_{j=1}^ia_i=A_i$ and $\sum_{j=1}^ib_i=B_i$.   Let $q_i(x)=R_{i}(x)-R_{i-1}(x)$ with $q_1(x)=R_1(x)$ and again note that $\sum_{j=1}^{i}q_j(x)=R_i(x)$.   Observe that $x_{i+1}=\arg\min_{x\in \S} \|x+\alpha_{i}(A_{i},B_{i})\|^2$.  Expanding the square and dropping terms that do not depend on $x$ it follows that $x_{i+1}=\arg\min_{x\in \S} \|x\|^2+2\alpha_{i}(A_{i},B_{i})^Tx=\arg\min_{x\in \S} R_{i}(x)+(A_{i},B_{i})^Tx = \arg\min_{x\in \S}\sum_{j=1}^{i} (q_j(x)+(a_j,b_j)^Tx) $.{ We conclude $ x_{i+1} \in \arg\min_{x\in \S} \sum_{j=1}^i r_j(x) $ for $r_j(x)=q_j(x)+(a_j,b_j)^Tx.$ Next we claim for all $u \in \S$ that   
\begin{align}
\sum_{j=1}^{i} r_j(x_{i+1}) \le \sum_{j=1}^{i} r_j(u).\label{eq:min}
\end{align}  We proceed by induction.  For $i=1$ we have $\sum_{j=1}^{i} r_j(x_{i+1}) = \sum_{j=1}^{1} r_j(x_2)$. Since $x_{2}$ minimises the right-hand side (\ref{eq:min}) holds. Now suppose   $\sum_{j=1}^{i-1} r_j(x_{j+1})\le \sum_{j=1}^{i-1} r_j(u)$.  Then
\begin{align}
\sum_{j=1}^{i} r_j(x_{j+1})\le r_{i}(x_{i+1})+\sum_{j=1}^{i-1} r_j(u)
\end{align}
This holds for all $u$, and so in particular for $u=x_{i+1}$.  Hence,
\begin{align}
\sum_{j=1}^{i} r_j(x_{j+1})\le \sum_{j=1}^{i} r_j(x_{i+1}) \stackrel{(a)}{\le} \sum_{j=1}^{i} r_j(u)
\end{align}
where $(a)$ follows how $ x_{i+1} \in \arg\min_{x\in \S} \sum_{j=1}^i r_j(x) $.   We conclude that $\sum_{j=1}^{i} r_j(x_{j+1})\le \sum_{j=1}^{i} r_j(u)$ for all $i=1,2,\dots$.}

Adding  $\sum_{j=1}^{i} r_j(x_{j})$ to both sides we get
\begin{align}
\sum_{j=1}^{i} (r_j(x_j) -r_j(u)) \le \sum_{j=1}^{i} (r_j(x_j)-r_j(x_{j+1}))
\end{align}
Substituting for $r_j(\cdot)$,
\begin{align}
\sum_{j=1}^{i} (q_j(x_j)-q_j(u))+\sum_{j=1}^{i}(a_j,b_j)^T(x_j-u)
\le \sum_{j=1}^{i} (q_j(x_j)-q_j(x_{j+1}))+\sum_{j=1}^{i}(a_j,b_j)^T(x_j-x_{j+1})
\end{align}
and rearranging,
\begin{align}
\sum_{j=1}^{i}(a_j,b_j)^T(x_j-u)
&\le \sum_{j=1}^{i} (q_j(u)-q_j(x_{j+1}))+\sum_{j=1}^{i}(a_j,b_j)^T(x_j-x_{j+1})
\end{align}
Now $\sum_{j=1}^{i} q_j(u)=R_i(u)$.  Also, $\sum_{j=1}^{i}q_j(x_{j+1})=R_1(x_2)+\sum_{j=2}^{i}q_j(x_{j+1})\ge R_1(x_2)$ {since $q_i(x) = (\frac{\sqrt{i}}{2} - \frac{\sqrt{i-1}}{2})\|x\|^2\ge 0$ for $i=2,\dots,i$.}   Hence, $-\sum_{j=1}^{i}q_j(x_{j+1})\le -R_1(x_2)$.  It follows that,
\begin{align}
\sum_{j=1}^{i}(a_j,b_j)^T(x_j-u)
&\le R_i(u)-R_1(x_2)+\sum_{j=1}^{i}(a_j,b_j)^T(x_j-x_{j+1})
\end{align}
\end{proof}

\bibliography{bibfile}

\begin{thebibliography}{13}
\providecommand{\natexlab}[1]{#1}
\providecommand{\url}[1]{\texttt{#1}}
\expandafter\ifx\csname urlstyle\endcsname\relax
  \providecommand{\doi}[1]{doi: #1}\else
  \providecommand{\doi}{doi: \begingroup \urlstyle{rm}\Url}\fi

\bibitem[Agarwal et~al.(2017)Agarwal, Luo, Neyshabur, and
  Schapire]{agarwal_corralling_2017}
Alekh Agarwal, Haipeng Luo, Behnam Neyshabur, and Robert~E. Schapire.
\newblock Corralling a {Band} of {Bandit} {Algorithms}.
\newblock In Satyen Kale and Ohad Shamir, editors, \emph{Proceedings of the
  30th {Conference} on {Learning} {Theory}, {COLT} 2017, {Amsterdam}, {The}
  {Netherlands}, 7-10 {July} 2017}, volume~65 of \emph{Proceedings of {Machine}
  {Learning} {Research}}, pages 12--38. PMLR, 2017.
\newblock URL \url{http://proceedings.mlr.press/v65/agarwal17b.html}.

\bibitem[Anderson and Leith(2019)]{anderson2019optimality}
Daron Anderson and Douglas Leith.
\newblock Optimality of the subgradient algorithm in the stochastic setting,
  2019.
\newblock URL \url{https://arxiv.org/abs/1909.05007}.

\bibitem[Cesa-Bianchi et~al.(2007)Cesa-Bianchi, Mansour, and
  Stoltz]{cesa-bianchi_improved_2007}
Nicolo Cesa-Bianchi, Yishay Mansour, and Gilles Stoltz.
\newblock Improved second-order bounds for prediction with expert advice.
\newblock \emph{Machine Learning}, 66\penalty0 (2-3):\penalty0 321--352, 2007.
\newblock \doi{10.1007/s10994-006-5001-7}.

\bibitem[Erven et~al.(2011)Erven, Koolen, Rooij, and
  Grunwald]{erven_adaptive_2011}
Tim~V. Erven, Wouter~M Koolen, Steven~D. Rooij, and Peter Grunwald.
\newblock Adaptive {Hedge}.
\newblock In J.~Shawe-Taylor, R.~S. Zemel, P.~L. Bartlett, F.~Pereira, and
  K.~Q. Weinberger, editors, \emph{Advances in {Neural} {Information}
  {Processing} {Systems} 24}, pages 1656--1664. Curran Associates, Inc., 2011.
\newblock URL \url{http://papers.nips.cc/paper/4191-adaptive-hedge.pdf}.

\bibitem[Freund and Schapire(1997)]{hedge97}
Yoav Freund and Robert~E. Schapire.
\newblock A decision-theoretic generalization of on-line learning and an
  application to boosting.
\newblock \emph{J. Comput. Syst. Sci.}, 55\penalty0 (1):\penalty0 119--139,
  1997.
\newblock \doi{10.1006/jcss.1997.1504}.

\bibitem[Gaillard et~al.(2014)Gaillard, Stoltz, and
  Erven]{gaillard_second-order_2014}
Pierre Gaillard, Gilles Stoltz, and Tim~van Erven.
\newblock A second-order bound with excess losses.
\newblock In Maria-Florina Balcan, Vitaly Feldman, and Csaba Szepesvari,
  editors, \emph{Proceedings of {The} 27th {Conference} on {Learning} {Theory},
  {COLT} 2014, {Barcelona}, {Spain}, {June} 13-15, 2014}, volume~35 of
  \emph{{JMLR} {Workshop} and {Conference} {Proceedings}}, pages 176--196.
  JMLR.org, 2014.
\newblock URL \url{http://proceedings.mlr.press/v35/gaillard14.html}.

\bibitem[Koolen and Erven(2015)]{koolen_second-order_2015}
Wouter~M. Koolen and Tim~van Erven.
\newblock Second-order {Quantile} {Methods} for {Experts} and {Combinatorial}
  {Games}.
\newblock In Peter Grunwald, Elad Hazan, and Satyen Kale, editors,
  \emph{Proceedings of {The} 28th {Conference} on {Learning} {Theory}, {COLT}
  2015, {Paris}, {France}, {July} 3-6, 2015}, volume~40 of \emph{{JMLR}
  {Workshop} and {Conference} {Proceedings}}, pages 1155--1175. JMLR.org, 2015.
\newblock URL \url{http://proceedings.mlr.press/v40/Koolen15a.html}.

\bibitem[Mourtada and Ga•ffas(2019)]{mourtada_optimality_2019}
Jaouad Mourtada and StŽphane Ga•ffas.
\newblock On the optimality of the {Hedge} algorithm in the stochastic regime.
\newblock \emph{J. Mach. Learn. Res.}, 20:\penalty0 83:1--83:28, 2019.
\newblock URL \url{http://jmlr.org/papers/v20/18-869.html}.

\bibitem[Sani et~al.(2014)Sani, Neu, and Lazaric]{ABprod}
Amir Sani, Gergely Neu, and Alessandro Lazaric.
\newblock Exploiting easy data in online optimization.
\newblock In \emph{Advances in Neural Information Processing Systems 27: Annual
  Conference on Neural Information Processing Systems 2014, December 8-13 2014,
  Montreal, Quebec, Canada}, pages 810--818, 2014.

\bibitem[Shalev-Shwartz(2012)]{Purple1}
Shai Shalev-Shwartz.
\newblock Online learning and online convex optimization.
\newblock \emph{Found. Trends Mach. Learn.}, 4\penalty0 (2):\penalty0 107--194,
  February 2012.
\newblock ISSN 1935-8237.
\newblock \doi{10.1561/2200000018}.
\newblock URL \url{http://dx.doi.org/10.1561/2200000018}.

\bibitem[Singla et~al.(2018)Singla, Hassani, and Krause]{singla_learning_2018}
Adish Singla, Seyed~Hamed Hassani, and Andreas Krause.
\newblock Learning to {Interact} {With} {Learning} {Agents}.
\newblock In Sheila~A. McIlraith and Kilian~Q. Weinberger, editors,
  \emph{Proceedings of the {Thirty}-{Second} {AAAI} {Conference} on
  {Artificial} {Intelligence}, ({AAAI}-18), the 30th innovative {Applications}
  of {Artificial} {Intelligence} ({IAAI}-18), and the 8th {AAAI} {Symposium} on
  {Educational} {Advances} in {Artificial} {Intelligence} ({EAAI}-18), {New}
  {Orleans}, {Louisiana}, {USA}, {February} 2-7, 2018}, pages 4083--4090. AAAI
  Press, 2018.
\newblock URL
  \url{https://www.aaai.org/ocs/index.php/AAAI/AAAI18/paper/view/16904}.

\bibitem[van Erven and Koolen(2016)]{metagrad}
Tim van Erven and Wouter~M. Koolen.
\newblock Metagrad: Multiple learning rates in online learning.
\newblock In \emph{Advances in Neural Information Processing Systems 29: Annual
  Conference on Neural Information Processing Systems 2016, December 5-10,
  2016, Barcelona, Spain}, pages 3666--3674, 2016.

\bibitem[Wintenberger(2017)]{wintenberger_optimal_2017}
Olivier Wintenberger.
\newblock Optimal learning with {Bernstein} online aggregation.
\newblock \emph{Machine Learning}, 106\penalty0 (1):\penalty0 119--141, 2017.
\newblock \doi{10.1007/s10994-016-5592-6}.

\end{thebibliography}

\end{document}